\documentclass[11pt]{article}
\usepackage{amsmath, amssymb, amsthm, mathtools}
\usepackage{algorithm}
\usepackage{algpseudocode}
\usepackage{geometry}
\usepackage{setspace}
\usepackage{hyperref}
\usepackage[shortlabels]{enumitem}
\usepackage{comment}
\usepackage{cleveref}
\usepackage{multicol}
\usepackage{tikz}
\usetikzlibrary{arrows.meta, calc, decorations.pathreplacing, decorations.markings, backgrounds}
\usepackage{xcolor}
\usepackage{float}
\usepackage{microtype}
\usepackage{booktabs}
\usepackage{tabularx}

\definecolor{manifoldgray}{RGB}{110,108,102}
\definecolor{manifoldfill}{RGB}{235,233,225}
\definecolor{vorfill}{RGB}{195,218,245}
\definecolor{vorstroke}{RGB}{24,95,165}
\definecolor{xcol}{RGB}{216,90,48}
\definecolor{ycol}{RGB}{24,95,165}
\definecolor{zcol}{RGB}{29,158,117}
\definecolor{geodark}{RGB}{80,78,74}
\definecolor{bisectcol}{RGB}{140,138,130}

\newcommand{\norm}[1]{\| #1 \|}
\newcommand{\C}{\mathcal{C}}
\newcommand{\E}{\mathbb{E}}
\newcommand{\M}{\mathcal{M}}

\renewcommand{\O}{\mathcal{O}}
\newcommand{\R}{\mathbb{R}}
\newcommand{\Z}{\mathbb{Z}}
\newcommand{\VR}{\mathrm{VR}}
\newcommand{\gesim}{\gtrsim}
\newcommand{\lesim}{\lesssim}
\newcommand{\eps}{\varepsilon}

\newcommand{\B}{\mathbb{B}}

\newcommand{\tA}{\widetilde{A}}
\newcommand{\tO}{\widetilde{O}}
\newcommand{\Na}{\mathbb{N}}

\theoremstyle{plain}
\newtheorem{theorem}{Theorem}[section]
\newtheorem{lemma}[theorem]{Lemma}
\newtheorem{proposition}[theorem]{Proposition}

\newtheorem{oracle}[theorem]{Oracle}

\theoremstyle{definition}
\newtheorem{definition}[theorem]{Definition}

\newtheorem{remark}[theorem]{Remark}
\newtheorem{claim}[theorem]{Claim}

\crefname{theorem}{Theorem}{Theorems}
\Crefname{theorem}{Theorem}{Theorems}
\crefname{lemma}{Lemma}{Lemmas}
\Crefname{lemma}{Lemma}{Lemmas}
\crefname{proposition}{Proposition}{Propositions}
\Crefname{proposition}{Proposition}{Propositions}
\crefname{corollary}{Corollary}{Corollaries}
\Crefname{corollary}{Corollary}{Corollaries}
\crefname{definition}{Definition}{Definitions}
\Crefname{definition}{Definition}{Definitions}
\crefname{remark}{Remark}{Remarks}
\Crefname{remark}{Remark}{Remarks}
\crefname{assumption}{Assumption}{Assumptions}
\Crefname{assumption}{Assumption}{Assumptions}
\crefname{oracle}{Oracle}{Oracles}
\Crefname{oracle}{Oracle}{Oracles}

\DeclareMathOperator{\Sec}{Sec}
\DeclareMathOperator{\Var}{Var}
\DeclareMathOperator{\diam}{diam}

\DeclareMathOperator{\polylog}{polylog}

\title{Recovering Geodesic Distances from Noisy Observations}

\author{
  Charles Fefferman\footnote{Department of Mathematics, Princeton University, Princeton, NJ 08544 USA. \texttt{\{cf, kr5621\}@math.princeton.edu}} \and Jonathan Marty\footnote{Program in Applied and Computational Mathematics (PACM), Princeton University, Princeton, NJ 08544, USA. \texttt{jm2866@math.princeton.edu}} \and Kevin Ren\footnotemark[1]
}

\date{}

\begin{document}

\maketitle


\begin{abstract}

Let $(\M,d,\mu)$ be a geodesic probability space and $X = \{x_i\}_{i \in [N]} \subset \M$ a finite $\mu$-i.i.d. sample. Suppose we observe noisy distance random variables $d'(x_i, x_j)$, related to the true distances. When can the true distances be recovered to small additive error? This question has attracted interest in two disparate fields: latent-space graph models and boundary-value inverse problems.

Our first result concerns the full-data case. In the case that $\M$ is a manifold, taking $N \asymp \eps^{-2d-2}|\log\eps|$ and using runtime $O(N^{2.5+o_d (1)})$, we recover all distances in $X$ to additive error $O(\eps |\log \eps|)$ with high probability, assuming $\E d'(x_i, x_j) = F(d(x_i,x_j))$ for an unknown, increasing bilipschitz $F$ and with the $d'(x_i, x_j)$ allowed to be arbitrarily heterogeneous in their distribution as long as they are bounded in subgaussian norm. This generalizes existing work, which required $F$ to be known and made strict assumptions on the distribution of $d'(x_i, x_j)$. Achieving this generalization required a novel ``ruler-building'' procedure that learns $F$ and a new concentration inequality. We further show that our results can be extended to a missing data case where we randomly lose access to noisy distances, with the chance increasing for larger distances, via a Dijkstra-type shortest-path algorithm.

Our second result uses ideas from the first to tackle a partial observability case: $X = Y \cup Z$ with $Z$ sampled from $\M$ and $Y$ sampled from an observer subset satisfying visibility conditions that model the assumptions made on the boundary in boundary-value inverse problems. In this setting, we recover all pairwise distances in $Z$ to additive error $O(\eps)$ with $|Z| \asymp \eps^{-2d-2}|\log\eps|$, $|Y| \asymp \eps^{-\max\{d_\nu,4\}}$ with $d_\nu$ determined by the volume-growth properties of the probability measure on the observer set, and runtime $O(|Z|^3)$. This is, to our knowledge, the first quantitative guarantee for this observer/observed split under noise, which complements existing stability results for boundary data.

\end{abstract}

\newpage

\tableofcontents

\newpage

\section{Introduction}
Let $(\M, d)$ be an unknown metric space. We have access to i.i.d samples $X = \{x_i\}_{i \in [N]}$ from a probability measure $\mu$ on $\M$, and we also have a sensor that can tell us distances between pairs in the sample. Unfortunately, the sensor is noisy and inaccurate: rather than the true distances $d(x_i, x_j)$, we observe noisy random variables $d'(x_i, x_j)$ whose relationship to the true distances is only partially known. This basic statistical question has arisen independently in both latent-space graph models and boundary-value inverse problems, fields that differ in both application and methodology.




Graphs determined 
by the position of nodes in some latent space have long been an object of study, based on applications in sensor networks, social networks, and biology \cite{DohertyPisterElGhaoui2001, HoffRafteryHandcock2002, HighamRasajskiPrzulj2008}. Recently, the recovery of latent-space representations from these graphs has attracted interest. The underlying graphs are random with nodes generated according to a distribution $\mu$ on the latent space satisfying concentration assumptions and with edges  whose probability of existence is $\phi(d(x,y))$ with $d$ the distance in the latent space and $\phi: \R_+ \to [0,1]$, termed the link function. Most papers assume the latent geometry is either Euclidean or spherical with unknown $\phi$, but make targeted assumptions on the function $\phi$ \cite{AriasCastroChannarondPelletierVerzelen2021, ArayaDeCastro2019, EldanMikulincerPieters2022}. The Euclidean paper \cite{AriasCastroChannarondPelletierVerzelen2021} has $\phi$ compactly supported in $[0,r]$ and only hopes to recover distances up to error $r$, while the spherical papers \cite{ArayaDeCastro2019, EldanMikulincerPieters2022} make assumptions on its spectrum that are specific to that setting. In either case, these assumptions allow the authors to bypass estimating $\phi$. Even more recently, \cite{huang2024reconstructing} recover the pairwise distances when the latent space is an embedded manifold with $\phi$ known and decreasing and one has access to the noisy Euclidean distances. 

The field of boundary-value inverse problems is concerned with the recovery of the manifold-geometry of a medium $\M$ from measurements taken at the boundary, with applications to medical imaging and geophysics \cite{anisotropy, stefanov2016boundary, hoskins2012principles, qu2015bent}. One such measurement is the travel-time of seismic waves between an earthquake's epicenter and a measurement station on the surface. These correspond directly to the geodesic distance between the epicenter and station in the wave-speed metric on the Earth. When the boundary and interior have ``good'' geometry with respect to each other, these boundary-interior distances uniquely determine the manifold geometry up to isometry \cite{LassasSharafutdinovUhlmann2003, KatsudaKurylevLassas2007}. However, we may only have access to noisy versions of the distances. This has prompted a line of study addressing this problem. Specifically, based on the work in \cite{geometricwhitney}, which shows that the pairwise distances between points in a dense net of $\M$ up to small additive errors determines the manifold up to small error in Gromov-Hausdorff distance, two papers have explored how to obtain this certificate from distinct input data, moving towards the full noisy boundary-interior distances setting. \cite{noisyintrinsic} works in the case where $(\M, d)$ is a manifold and we are given access to distances with sub-exponential additive noise of mean $0$ and known constant variance. 
\cite{noisyintrinsic} also recovers all short pairwise distances in the case where a subset of the distances are randomly missing, with the chance of missingness trending upward as points are further from each other. Meanwhile, \cite{fefferman2025reconstruction} considers an observer region $S \subset \M$ which contains an $R$-ball for some $R = R(\M) > 0$, and shows that if you sample an observer set $Y$ from $S$ and a receiver set $Z \subset \M$ and have access to the distances $d(y_i, z_j)$ up to small error, then you can recover $d(z_i, z_j)$ up to small error.


Our discussion in the previous paragraphs hints at the following general problem: let $(\M, d)$ be a metric space and $X = \{x_i\}_{i \in I} \subset \M$. Suppose we are given noisy versions of the distances $d'(x_j, x_k)$ for a subset of the pairs in $O \subseteq X \times X$. Is it possible to recover the pairwise distances in $P \subseteq X \times X$ to small additive error? 

In this paper we address this question when $\M$ is a geodesic space and $X$ is sampled from a nice probability measure $\mu$ on $\M$. Our first result, announced previously in \cite{FMR2025}, generalizes the setting of \cite{noisyintrinsic}. Let $(\M,\mu,d)$ denote a diameter 1 geodesic probability space and let $X_1, \dots, X_N$ be $\mu$-i.i.d. samples. We observe noisy-distance random variables $d'(X_i, X_j)$, which are allowed to be as heterogeneous as desired as long as their tails are subgaussian with a uniform bound, and which satisfy $\E d'(X_i, X_j) = F(d(x,y))$ with $F$ increasing, bilipschitz and \textbf{unknown}. This is a strictly more difficult noise regime than previous work, and the heterogeneity and unknown $F$ introduce difficult conceptual and technical challenges. For $\eps > 0$ and assuming independence of the noisy distances and mild conditions on their distributions, we develop a novel polynomial-time algorithm that recovers all distances in an $O(\eps)$-dense subsample of $\M$ to accuracy $\tO(\eps)$ with sample size and runtime both polynomial in $\frac{1}{\eps}$, with $\tO(f) = O(f \cdot \polylog (f))$. We're also able to handle the missing data setting explored in \cite{noisyintrinsic}; in fact, we go further and recover all distances. This is the first method that handles unknown $F$ in a general setting, and it does so by learning the function via a novel ``ruler-building'' procedure. The implications of this method for unknown link functions in the graph setting is clear. However, not knowing the precise relationship between a measured value and the variable underlying it is a common statistical situation and an incorrect choice of model can yield misleading results \cite{DGMY2012, ElzanatyEtAl2024, FournierEtAl2013}. In fact, this shows up in inverse problems \cite{FournierEtAl2013, FaucherScherzer2023}, making the ability to handle unknown $F$ a valuable piece of functionality. Further, our methods are proved without techniques specific to Riemannian geometry, opening the door to a broad class of geodesic probability spaces which satisfy conditions characteristic of manifolds with sectional curvature bounds, including Alexandrov spaces, stratified spaces, and metric trees.

Our second result capitalizes on the new concepts developed in the proof of the first to tackle the full noisy boundary-interior distances setting (in fact, a generalization of it): our $X = Y \cup Z$ with $Y$ and $Z$ i.i.d. samples from distributions on distinct spaces. Specifically, points in $Z$ are sampled from a geodesic metric space $\M$ and $Y$ is sampled from a subset which satisfies certain visibility axioms with respect to $\M$. These visibility axioms are analogous to the sufficient assumptions established in the inverse problems literature for recovering a manifold from boundary observations \cite{Michel1981, KatsudaKurylevLassas2007, IlmavirtaEtAl2025}. In fact, this is the first instance of the recovery of interior distances from noisy boundary-interior distances. Further, its setup is distinct from the $R$-ball observer region considered in \cite{fefferman2025reconstruction}, providing a new branch of quantitative literature.

It should be noted that some of our techniques overlap with those of the earlier ArXiv posting \cite{huang2025reconstructing}, which we were unaware of when we posted our original preprint \cite{FMR2025}. More details on this are provided as part of \Cref{section:comparison} and an announcement of work by both teams following the initial preprints is in \Cref{subsec:announcement}.

\subsection{Results and Interpretation}
\label{sec:mainresults}

We now present our results along with high-level intuition for the assumptions made, deferring technical details to \Cref{section:intuition}. 

\begin{theorem}
\label{thm:main}
\label{thm:mainresthm1}
    Let $(\M,d,\mu)$ be a geodesic probability space with diameter $1$ and $\mu$ satisfying (A1)--(A2) below. Let $X$ be an $N$ point i.i.d. sample from $\mu$. Suppose that for each unordered pair $\{x,y\}$ with $x,y \in X$ distinct we have access to a random variable $d'(x,y)$ satisfying (B1)--(B2) below, with $\E d'(x,y)=F(d(x,y))$,
    where $F:[0,1]\to\R$ is an \textbf{unknown}, increasing, $C$-bilipschitz function and $|F(t)|\le C_2$ for all $t\in[0,1]$. Then there is a constant $K=K(d_\mu,c,C,C_1,C_2)$ such that, for every $0<\eps\le 2^{-20}$ and $0<\theta<1$, if
    \[
        N \ge K \kappa^{-d_\mu} \eps^{-2d_\mu-2}\left(d_\mu\log\frac1\eps+\log\frac1\theta\right),
    \]
    then with probability at least $1-\theta$ we can recover the geodesic distances $d(x,y)$ for all $x,y\in X$ up to additive error $O(\eps\log\eps^{-1})$ in time $O(N^{2.5 + o_{d_\mu} (1)})$.
    \begin{enumerate}[label=(A\arabic*)]
        \item $\forall x \in \M, r > 0, \mu(B(x,r)) \geq c r^{d_\mu}$.
        \item $\forall x,y \in \M, d(x, y) \ge 4\eps \implies \mu\left(\left\{z\in\M:|d(x,z) - d(y,z)| \ge \frac{d(x,y)}{4}\right\}\right) \geq \kappa$.
    \end{enumerate}

    \begin{enumerate}[label=(B\arabic*)]
        \item The noisy distances $\{d'(x,y)\}_{\{x,y\}\in X}$ are mutually independent, conditioned on the sampled set $X$.
        \item $\forall x, y \in \M$, $d'(x,y)$ is sub-Gaussian with $\|d'(x,y)\|_{\psi_2} \leq C_1$; see \Cref{sec:prelim} for the definition of Orlicz norm.
    \end{enumerate}
\end{theorem}

\begin{remark}\label{rem:true for manifolds}
    Assumption (A2) is always true with $\kappa = c_1 \eps^{d_\mu}$ since the set on the LHS contains $B(y, \eps)$. However, this inflates $N \geq \tO(\eps^{-d_\mu^2-2d_\mu-2})$. For manifolds, (A1) and (A2) are true with $c, \kappa$ being constants depending on manifold parameters; see Proposition \ref{prop:kappa exists}. Assumptions (B1)-(B2) can be slightly weakened based on the requirements of Proposition \ref{prop:hoeffding ip}.
\end{remark}

\begin{remark}
    In this paper, we sometimes drop the $\mu$ and write $d$ in place of $d_\mu$; context should make it clear whether $d$ is the dimension or the distance function.
\end{remark}

The conditions on the geodesic probability space are characteristic of manifolds with sectional curvature and injectivity radius bounds, although the proofs use only the explicit lower-volume and one-sided witness assumptions (A1)--(A2). Thus the same arguments apply verbatim to any finite-diameter geodesic probability space satisfying these two hypotheses, including some non-smooth examples where lower ball-volume bounds and distance witnesses are available. We do not pursue a separate taxonomy of such spaces here. Further, the conditions (B1)--(B2) are standard for noisy random variables. Subgaussian variables have tails dominated by a Gaussian's. For instance, any bounded random variable is subgaussian. Although standard, the independence assumptions (B1) would likely need to be relaxed to enable practical application of our methods.

In order to understand the $\E d'(x,y) = F(d(x,y))$ condition, let's work through a few real-world examples. If we have a wave emanating from one point in a medium, the travel time to arrive at another point should be proportional to the distance in the wave-speed metric, hence $\E d'(x,y) = c d(x,y)$ with $c > 0$ is a good model. By allowing for a bilipschitz distortion of the distance via $F$, we are able to account for situations where the measured variable isn't directly proportional to the distance, but depends monotonically on it, improving the expressivity of our method. For example, in the scenario of sensor localization, one node might send a signal and the other read off the received amplitude, which degrades with distance traveled. By swapping the sign we obtain an increasing, bilipschitz $F$.

The diameter 1 assumption in \Cref{thm:mainresthm1} is non-standard. In general, our methods are applicable to geodesic metric spaces of finite diameter. However, the ambiguity in the definition of $F$ makes it impossible to recover the scale from the provided data. Case in point: suppose that $F:[0,1]\to \R$ is $C/2$-bilipschitz and let $ 1/2 \leq \ell \leq 2$. Noting that $F(d(x,y)) = F(\ell \frac{1}{\ell} d(x,y))$, we can replace $F(t)$ with $F(\ell t)$ and $d$ with $\frac{1}{\ell}d$ in our setup and the data is the same. This preserves the $C$-bilipschitzness but divides the diameter by $\ell$. In order to resolve this ambiguity we pick diameter 1 as our scale. For a discussion of how one might recover the scale with additional data, see \Cref{sec:recscale}.


By utilizing the basic ideas of our proofs, we're able to provide two extensions which move our framework closer to practicality and generalize the setups of \cite{fefferman2025reconstruction, noisyintrinsic}. The first concerns situations where some proportion of the noisy distances are ``missing.'' For example, in sensor localization tasks some of the sensors might fail to return signals.

\begin{theorem}[Missing observations]
\label{thm:mainresthm2}
    Assume we have the same setup as in Theorem \ref{thm:mainresthm1} with the following change: Let $m(x,y)$ be a Bernoulli random variable which takes value $1$ with probability $p(x,y)$, independent of the noisy-distance variables and independent across unordered pairs. Suppose we have access to $d''(x,y)$, which equals $d'(x,y)$ if $m(x,y)=1$ and equals $\mathrm{MISSING}$ otherwise. Let $r_0, \phi, \lambda_1, \lambda_2 > 0$, and assume the stronger assumption $(A2')$ and (C1)--(C2) below:
    \begin{enumerate}[label=(A2')]
        \item $\forall x,y\in\M$ with $d(x,y)\le r_0/3$, $\mu\left(\left\{z\in B(x,r_0/2):|d(x,z)- d(y,z)|\ge\frac{d(x,y)}4\right\}\right) \ge \kappa$.
    \end{enumerate}
    \begin{enumerate}[label=(C\arabic*)]
        \item For any distinct $x, y, z, w \in X$ such that $d(z, w) < d(x, y) + \lambda_2 p(x, y)$, we have $p(z, w) > \lambda_1 p(x, y)$. 
        \item If $d(x,y)\le r_0$, then $p(x,y)\ge \phi.$
    \end{enumerate}
    Then if $N\ge K\eps^{-2d_\mu-2}
        \left(d_\mu\log\frac1\eps+\log\frac1\theta\right)$,with probability at least $1-\theta$ we can recover the geodesic distances $d(x,y)$ for all $x,y\in X$ up to additive error $O(\frac{\eps}{r_0^2} \log \eps^{-1})$ in time $O(N^{4 + o_d (1)})$. Here $K = K(C, C_1, C_2, \kappa, c_1, c_2, c, d_\mu, \phi, \lambda_1, \lambda_2)$ and $0 < \eps \leq \eps_0(r_0, \phi, \lambda_1, \lambda_2)$.
\end{theorem}

The construction of the probability of missingness $p(x,y)$ is such that there is a guarantee that a sizeable portion of distances of size $\leq r_0$ will be present (C2), a reasonable physical assumption. ($A2'$) is then a specialization of (A2) to this domain. Finally, (C1) ensures that the missing data model is non-trivial, and generalizes the framework used in \cite{noisyintrinsic}\footnote{In \cite{noisyintrinsic}, (C1) is replaced by $c_1 p_1(d(x,y)) \leq p(x,y) \leq c_2 p_1(d(x,y))$ with $0 < c_1 < 1 < c_2$ and $p_1$ non-increasing with $p_1(0) = h > 0$ and $\|p_1\|_{C^1([0,1])} \leq H$; one can check that \cite{noisyintrinsic}'s condition implies our version of (C1).}. With this setup, we have the ability to do statistics with local distances, which will be fleshed out in \Cref{section:intuition}.

The next extension is oriented towards cases where we only have access to the noisy distance variables between an ``observer'' and ``observed'' region. This occurs practically in both seismology and medical imaging contexts, where probes on the boundary (our observer region) recover signals that track their distance to interior points.

\begin{theorem}[Partial observability]
\label{thm:mainresthm3}
    Let $(\M,d,\mu)$ be a diameter-one geodesic probability space satisfying (A1)--(A2). Let $Y\subseteq \M$ be an observer set equipped with a probability measure $\nu$ satisfying (D1)--(D4) below. Let $X_\mu$ be $N_\mu$ i.i.d. draws from $\mu$ and let $X_\nu$ be $N_\nu$ i.i.d. draws from $\nu$. Suppose we have access to independent random variables $d'(x,y)$ for all $(x,y)\in X_\mu\times X_\nu$ satisfying (B2) and $\E d'(x,y)=F(d(x,y))$, where $F:[0,1]\to\R$ is unknown, increasing, $C$-bilipschitz, and satisfies $|F|\le C_F$. Then there are constants $K=K(d,d_\nu,c,c_\nu,r_0,\kappa,C,C_1,C_F)$ and $\eps_0=\eps_0(d_\mu,d_\nu,c,c_\nu,r_0,\kappa,C,C_1,C_F)>0$ such that for every $0<\eps\le\eps_0$ and $0<\theta<1$, if
    \begin{flalign*}
        N_\mu \ge K\eps^{-2d-2}\left((d_\mu+d_\nu+1)\log\frac1\eps+\log\frac1\theta\right) && N_\nu \ge
        K\eps^{-\max\{d_\nu,4\}}\left((d_\mu+d_\nu+1)\log\frac1\eps+\log\frac1\theta\right)
    \end{flalign*}
    then with probability at least $1-\theta$ we can recover the geodesic distances $d(x_1,x_2)$ for all $x_1,x_2\in X_\mu$ up to additive error $O(\eps)$ in time $O(N_\mu^3 + N_\mu^2 N_\nu)$.

\begin{enumerate}[label=(D\arabic*)]
    \item For all distinct $x_1,x_2\in\M$, $\nu\left(\left\{y\in Y:
        |d(x_1,y)-d(x_2,y)|>\frac{d(x_1,x_2)}4
        \right\}\right)\ge \kappa>0$.
    \item For all $x_1,x_2\in\M$, there exists $z\in Y$ such that $d(x_1,x_2)=|d(x_1,z)-d(x_2,z)|$.
    \item For every $0\le r\le 1$, there exist $y_1,y_2\in Y$ such that $d(y_1,y_2)=r$.
    \item There are constants $c_\nu>0$ and $d_\nu>0$ such that for every $y\in Y$ and every $0<r\le1$, $\nu(B_\M(y,r)\cap Y)\ge c_\nu r^{d_\nu}$.
\end{enumerate}
\end{theorem}

\begin{remark}(D2) closely parallels the set of sufficient conditions for the distances between $Y$ and $\M$ as compact length spaces to specify $\M$ within small Gromov-Hausdorff distance determined in \cite{IlmavirtaEtAl2025}. For details, see Proposition 12 and Theorem 6 therein.
\end{remark}

These conditions are analogous to the geometric visibility hypotheses that appear in boundary distance and travel-time inverse problems: boundary measurements must distinguish interior points, and boundary points must provide enough ``ruler'' pairs to calibrate distances. Understanding their role is best done through the case of $\M = \B^d$ and $Y = \partial\B^d$ with $\mu$ and $\nu$ uniform on their respective spaces and $d$ the flat metric. In this case, (D1) tracks the existence of a region of nonzero probability in $Y$ that can witness the distance between $x_1$ and $x_2$. To see this in our test case, let $x_1, x_2 \in \B^d$ with $|x_1 - x_2| > 2\ell$ and draw a line through them. Let $y$ be a point at which the line intersects $\partial \B^d$. A small neighborhood of $y$ will satisfy $|x_1-y| - |x_2 - y| > \ell$. Further, using $y$ we can recover the distance between $x_1$ and $x_2$: $|d(x_1,y) - d(x_2,y)| = d(x_1,x_2)$, which is an instance of (D2). In addition, (D3) is a sufficient condition to extend the ruler construction to this domain, and its application will be explained in \Cref{sec:obsreg}. Finally, (D4) is a concentration condition necessary to obtain a good sample with high probability. In our case, we could pick $d_\nu = d-1$.

\subsection{Methods and Intuition}
\label{section:intuition}

\subsubsection{Reduction to Finding Clusters}

We adopt a standard tactic of first recovering local clusters of sample points. Suppose for each $x \in X$ we can find a cluster $C(x)$ with small enough radius and enough points, then with high probability we have the following identity $$\E d'(x,y) \approx \frac{1}{|C(x)||C(y)|}\sum_{x' \in C(x), y' \in C(y)} \E d'(x',y') \approx \frac{1}{|C(x)||C(y)|}\sum_{x' \in C(x), y' \in C(y)} d'(x',y').$$
The first approximation follows since $d(x', y') \approx d(x, y)$ whenever $x'$ is close to $x$ and $y'$ is close to $y$, and the function $\E d'(x, y) = F(d(x, y))$ is $C$-lipschitz, and the second approximation follows since the average of independent random variables is close to its expectation by a Hoeffding-type concentration inequality.

\subsubsection{Inner Products as Proximity Sensors}
\label{subsection:innerprodsproxsensors}
From the previous subsection, we have reduced the problem of finding expected noisy distances to the problem of identifying small-radius clusters around a given sample point $x$.
A key impediment to identifying clusters is obtaining an indicator of proximity between sample points. When given the raw $d(x,y)$, this is trivial, but the $d'(x,y)$ can have subgaussian terms which introduce large biases, concealing adjacency. We mitigate this issue by calculating ``inner product'' random variables for $x \neq y$ (with $Z$ a finite set of $z \sim \mu$ i.i.d.)

\begin{equation}
\label{eq:innerprodsproxsensors1}
I(x,y) = \frac{1}{|Z|}\sum_{z \in Z} d'(x,z)d'(z,y)
\end{equation}
Let $f_x = \E d'(x,\cdot) = F(d(x,\cdot)): \M \to \R$, which by our assumptions on $F$ is a bilipschitz function of the intrinsic metric. On a theoretical note, the $f_x$ are analogous to the image of $x \in M$ under the Kuratowski embedding,  but in place of the pointed distance function $y \mapsto d(x,y)$, we use $y \mapsto F(d(x,y))$. Since $\M$ has finite diameter, the $f_x$ are bounded and hence in $L^2(\M,\mu)$. Thus, (B1) ensures that $\E I(x,y) = \langle f_x, f_y \rangle$, with $\langle \cdot, \cdot \rangle$ the inner product on $L^2(\M,\mu)$. Further, (B2) ensures that $I(x,y)$ is close to its expectation w.h.p. when $|Z|$ is large. This follows from the $d'(x,z)$ being subgaussian implying that $d'(x,z)d'(y,z)$ are subexponential. However, this product structure yields a new Hoeffding-type inequality (see \Cref{prop:hoeffding ip}), which allows us to show that $I(x,y)$ concentrates for large $|Z|$ as if it were Gaussian. 

Our algorithm relies on the above inner product as well as a novel optimization procedure given below in the abstract to approximate $\|f_x - f_y\|_2^2$, where $X$ is a $\kappa$-net of $\M$ and $C$ is a constant.

\begin{equation}
\label{eq:proxsensbounds}
    \frac{1}{2}\|f_x -f_y\|_2^2 - C\kappa\leq \sup_{z \in X, z \neq x,y} |\langle f_x - f_y, f_z \rangle| \leq C_F\|f_x-f_y\|_2
\end{equation}
 The LHS follows from $|\langle f_z,f_x-f_y\rangle| \geq \frac{1}{2}\|f_x - f_y\|_2^2$ for $z = x$ or $y$ and the fact that $X$ is a $\kappa$-net and the $f_x$ are Lipschitz, and the RHS from $|F(t)| \leq C_F$ for all $t \in [0,1]$. 
 
 The  quantity $\|f_x-f_y\|_2^2$, approximated above, acts as a proxy for $d(x,y)$. To see this, observe that taking $d(x,y) > 4\eps$ means that on an $\eps$-ball around $x$, $|f_x - f_y|$ is above $\frac{2 \eps}{C}$, with $C$ the lower Lipschitz constant of $F$. Lower bounding the measure of this ball, which follows from (A1), gives a lower bound on $\|f_x - f_y\|^2_2$. Alternatively, for any given $\eps > 0$, setting $d(x,y)$ sufficiently small means that $|f_x - f_y| < C\eps$ everywhere on the manifold (since F is Lipschitz). Thus, $\|f_x - f_y\|_2^2$ acts as a proximity sensor. Tighter lower bounds on $\|f_x - f_y\|_2^2$ for $d(x,y) > 4 \eps$ can be achieved based on concentration bounds on the metric cone (A2), which our algorithm utilizes.

\subsubsection{Using Rulers to Recover the Bilipschitz Function}

Having recovered the pairwise average noisy distances $F(d(x,y))$ using cluster-cluster distances, we learn an approximation of $F^{-1}$, termed $\widetilde{F}^{-1}$, and use it to recover the pairwise distances $d(x,y)$ via $\tilde{d}(x,y) = \widetilde{F}^{-1}(F(d(x,y)))$. In order to do so, we introduce the idea of a ``ruler'': a set of points $x=x_0, \dots ,x_N=y$ that approximate the geodesic between $x$ and $y$. Namely, the $d(x_i, x_{i+1})$ are approximately uniform and $\sum_{i \in [N]} d(x_{i-1}, x_i) \approx d(x,y)$. Given a ruler we can construct $\widetilde{F}(\frac{i}{N}) = F(d(x_0,x_i))$, which can be inverted via $\widetilde{F}^{-1}(s) = \max\{q/N : q \in [N], \widetilde{F}(\frac{q}{N}) < s\}$. In the case that $x$ and $y$ are unit distance from each other, the resultant ruler allows us to invert all the pairwise average noisy distances and recover the $\tilde{d}(x,y)$.

We introduce a novel method to build these rulers. The crux of the procedure is taking geodesic midpoints $z = M(x,y)$ such that $d(x,z), d(y,z) \approx \frac{1}{2}d(x,y)$. This relies on the fact that, while $F(d(x,y))$ conceals the true values of the distances, it gives us access to their order. Namely, $F(d(x,y)) < F(d(x',y')) \implies d(x,y) < d(x',y')$. Let $W$ denote the set of points $z$ where $F(d(y,z)) < F(d(x,z)) \implies d(y,z) < d(x,z)$, the Voronoi section of y with respect to the point-set $\{x,y\}$. Taking $z^* \in W$ that minimizes $F(d(x,z))$ forces $z^*$ to be a midpoint in the sense described above. This is illustrated in \Cref{fig:bisect}.

To build the ruler, we start by finding $z_0,z_1 \in \M$ that maximize $F(d(x,y))$, which gives us a pair of points approximately unit distance from each other. Next, we apply our midpoint oracle to obtain $z_{1/2} = M(z_0,z_1)$. From there we set $z_{1/4} = M(z_0,z_{1/2})$, $z_{3/4} = M(z_{1/2}, z_1)$ and the process continues recursively until after k recursions we have access to $z_{q/2^k}$ for all $q \in [2^k]$, all satisfying $d(x_0,z_{q/2^k}) \approx \frac{q}{2^k}d(z_0,z_1)$. Since we only have approximate access to the $F(d(x,y))$, this binary procedure corresponds to logarithmic error propagation.

\begin{figure}[H]
\centering
\begin{tikzpicture}[
    >=Stealth,
    dot/.style={circle, fill, inner sep=2.2pt},
    geodesic/.style={
      thick,
      decoration={
        markings,
        mark=at position 0.35 with {\arrow{Stealth}},
        mark=at position 0.72 with {\arrow{Stealth}}
      },
      postaction={decorate}
    },
    halfarc/.style={dashed, thick},
    leader/.style={dashed, thin, opacity=0.55},
    scale=0.8
]
 
\begin{scope}
  \fill[manifoldfill]
    (-4.0, 0.0)
    .. controls (-4.0, 1.6)  and (-2.8, 2.4)  .. (-1.9, 2.3)
    .. controls (-1.1, 2.2)  and (-0.7, 1.55) .. (-0.35,1.05)
    .. controls ( 0.0, 0.55) and ( 0.2, 0.45) .. ( 0.55, 0.5)
    .. controls ( 0.9, 0.55) and ( 1.1, 0.55) .. ( 1.45,1.05)
    .. controls ( 1.8, 1.55) and ( 2.2, 2.2)  .. ( 3.0, 2.3)
    .. controls ( 3.8, 2.4)  and ( 4.5, 1.6)  .. ( 4.5, 0.0)
    .. controls ( 4.5,-1.6)  and ( 3.8,-2.4)  .. ( 3.0,-2.3)
    .. controls ( 2.2,-2.4)  and ( 1.8,-1.55) .. ( 1.45,-1.05)
    .. controls ( 1.1,-0.55) and ( 0.9,-0.55) .. ( 0.55,-0.5)
    .. controls ( 0.2,-0.45) and ( 0.0,-0.55) .. (-0.35,-1.05)
    .. controls (-0.7,-1.55) and (-1.1,-2.2)  .. (-1.9,-2.3)
    .. controls (-2.8,-2.4)  and (-4.0,-1.6)  .. cycle;
\end{scope}
 
\begin{scope}
  \clip
    (-4.0, 0.0)
    .. controls (-4.0, 1.6)  and (-2.8, 2.4)  .. (-1.9, 2.3)
    .. controls (-1.1, 2.2)  and (-0.7, 1.55) .. (-0.35,1.05)
    .. controls ( 0.0, 0.55) and ( 0.2, 0.45) .. ( 0.55, 0.5)
    .. controls ( 0.9, 0.55) and ( 1.1, 0.55) .. ( 1.45,1.05)
    .. controls ( 1.8, 1.55) and ( 2.2, 2.2)  .. ( 3.0, 2.3)
    .. controls ( 3.8, 2.4)  and ( 4.5, 1.6)  .. ( 4.5, 0.0)
    .. controls ( 4.5,-1.6)  and ( 3.8,-2.4)  .. ( 3.0,-2.3)
    .. controls ( 2.2,-2.4)  and ( 1.8,-1.55) .. ( 1.45,-1.05)
    .. controls ( 1.1,-0.55) and ( 0.9,-0.55) .. ( 0.55,-0.5)
    .. controls ( 0.2,-0.45) and ( 0.0,-0.55) .. (-0.35,-1.05)
    .. controls (-0.7,-1.55) and (-1.1,-2.2)  .. (-1.9,-2.3)
    .. controls (-2.8,-2.4)  and (-4.0,-1.6)  .. cycle;
  \fill[white]
    (0.60, 3.0) -- (5.5, 3.0) -- (5.5,-3.0) -- (0.60,-3.0)
    .. controls (0.35,-2.5) and (0.35,2.5) .. cycle;
  \fill[vorfill, opacity=0.55]
    (0.60, 3.0) -- (5.5, 3.0) -- (5.5,-3.0) -- (0.60,-3.0)
    .. controls (0.35,-2.5) and (0.35,2.5) .. cycle;
  \foreach \y in {-2.8,-2.4,...,3.0}{
    \draw[vorstroke, opacity=0.13, line width=0.4pt]
      (0.50,\y) -- (5.5, \y+0.55);
  }
\end{scope}
 
\draw[manifoldgray, line width=1.2pt]
    (-4.0, 0.0)
    .. controls (-4.0, 1.6)  and (-2.8, 2.4)  .. (-1.9, 2.3)
    .. controls (-1.1, 2.2)  and (-0.7, 1.55) .. (-0.35,1.05)
    .. controls ( 0.0, 0.55) and ( 0.2, 0.45) .. ( 0.55, 0.5)
    .. controls ( 0.9, 0.55) and ( 1.1, 0.55) .. ( 1.45,1.05)
    .. controls ( 1.8, 1.55) and ( 2.2, 2.2)  .. ( 3.0, 2.3)
    .. controls ( 3.8, 2.4)  and ( 4.5, 1.6)  .. ( 4.5, 0.0)
    .. controls ( 4.5,-1.6)  and ( 3.8,-2.4)  .. ( 3.0,-2.3)
    .. controls ( 2.2,-2.4)  and ( 1.8,-1.55) .. ( 1.45,-1.05)
    .. controls ( 1.1,-0.55) and ( 0.9,-0.55) .. ( 0.55,-0.5)
    .. controls ( 0.2,-0.45) and ( 0.0,-0.55) .. (-0.35,-1.05)
    .. controls (-0.7,-1.55) and (-1.1,-2.2)  .. (-1.9,-2.3)
    .. controls (-2.8,-2.4)  and (-4.0,-1.6)  .. cycle;

 
%
 
\draw[xcol, line width=1.5pt, dashed,
      postaction={decorate},
      decoration={
        markings,
        mark=at position 0.52 with {\arrow[xcol]{Stealth}}
      }]
  (-2.7,-0.55)
  .. controls (-1.85, -0.05) and (-0.55, 0.10) ..
  (0.60, 0.0);
 
\draw[ycol, line width=1.5pt, dashed,
      postaction={decorate},
      decoration={
        markings,
        mark=at position 0.50 with {\arrow[ycol]{Stealth}}
      }]
  (0.60, 0.0)
  .. controls (1.55, 0.15) and (2.45, -0.05) ..
  (3.0, -0.45);
 
\node[xcol, font=\small] at (-1.08, 0.28) {$d(x,z)$};
\node[ycol, font=\small] at ( 1.92, 0.28) {$d(z,y)$};
 
\node[dot, xcol, label={[xcol, font=\normalsize, xshift=-3pt, yshift=-1pt]below left:$x$}]
  (X) at (-2.7, -0.55) {};
 
\node[dot, ycol, label={[ycol, font=\normalsize, xshift=3pt, yshift=-1pt]below right:$y$}]
  (Y) at ( 3.0, -0.45) {};
 
\node[dot, zcol, inner sep=2.8pt,
      label={[zcol, font=\normalsize, yshift=0.0pt]above right:$z$}]
  (Z) at (0.55, 0.0) {};
 
\node[vorstroke, font=\small, align=center] at (2.85, 2.75)
  {$\{w \in \mathcal{M} : d(w,y) < d(w,x)\}$};
 
\node[zcol, font=\small, align=center] (Zlabel) at (0.62, -3.05)
  {$z\ =\ \underset{\substack{w\,\in\,\mathcal{M},\; d(w,y)<d(w,x)}}{\arg\min}\ d(x,w)$};
 
\node[manifoldgray, font=\small\itshape] at (-3.3, 2.75) {$\mathcal{M}$};
 
\end{tikzpicture}
\caption{%
  \textbf{Midpoint oracle}.}
\label{fig:bisect}
\end{figure}

\subsubsection{Dealing with Missing Distances}
\label{sec:dealmissdist}

In order to extend our methods to deal with missing distances, we need to adapt both the inner product and ruler frameworks. When we take the inner products of \Cref{eq:innerprodsproxsensors1} using $d''$ in place of $d'$, we end up multiplying each summand by $m(x,z)m(z,y)$. Hence in place of the inner product in \Cref{eq:proxsensbounds}, we obtain $\langle f_xp_x -f_yp_y, f_zp_z\rangle$, with $p_w = p(w,\cdot)$. By (C1), the $p_w$ can fail to be Lipschitz, so it is impossible to lower bound the magnitude of this expression. Intuitively, a fix would be to instead calculate $\langle f_x - f_y, f_z p_x p_y p_z\rangle$, but this runs into the issue that 

\[
    \langle f_x - f_y, f_z p_xp_yp_z\rangle - \langle f_x - f_y, f_w p_xp_yp_w\rangle = \langle f_zp_z - f_w p_w, (f_x - f_y)p_x p_y\rangle
\]
which forbids the logic used to produce the lower bound for \Cref{eq:proxsensbounds}. This issue is solved by

\begin{equation}
    \sup_{\substack{v,w \in X\\ |\{v,w,x,y\}| = 4}} |\langle f_x - f_y, (f_v - f_w) p_x p_y p_v p_w\rangle|
\end{equation}
which is bounded analogously to \Cref{eq:proxsensbounds}. The lower bound follows from $v \in X$ close to $x$ and $w \in X$ close to $y$ and the upper bound from the boundedness of the functions involved.  For $x,y$ that are close to each other, (A2') and (C2) together ensure that the above average includes enough summands in regions of $z$ that discern $f_x$ from $f_y$. If there are not enough summands, this is detectable from the size of $\sum_{z \in X} m(x,z) m(y,z)$.

We can then construct clusters and use cluster-cluster distances to recover $F(d(x,y))$. Based on (C1) we have enough pairwise distances to do so for all $d(x,y) \leq r_0$, giving us a graph of recovered distances $G$ with vertices $X$ (our $\kappa$-net from before.) Now, given that any number $s \geq \frac{1}{3}$ can be expressed as $n\ell$ with $n \in \Na$ and $\ell \in [1/3, 2/3]$, for $d(x,y) \ge \frac{r}{3}$ there exists a sequence of equally spaced points $x = x_0, \dots,x_n = y$ with $d(x,y) = \sum_{i=1}^n d(x_{i-1}, x_i)$ and $r/3 \leq d(x_{i-1}, x_i) \leq 2r/3$, so that $n = \frac{d(x,y)}{d(x_0, x_{1})}$. In order to access equally spaced sequences of points, we introduce a directed graph $H$ with its vertices all pairs of points $(x,y) \in X^2$. For $x,y,z$ distinct, there is an edge from $(x,y)$ to $(y, z)$ if $F(d(x,y)) > F(d(y,z))$. Observe that this means any path in $H$ corresponds to a path $x_0 \to \dots \to x_n$ in $G$ with $i \mapsto d(x_i, x_{i+1})$ a decreasing function. For each $x \in X$, by a simple optimization routine we pick $\tau(x)$ such that $d(x,\tau(x))$ is comparable to $r_0$ with mild absolute constants less than 1. Since we can take midpoints between points within distance $r_0$ of each other, we can construct a ruler between $x$ and $\tau(x)$, denoted $(z_q)_{q \in I_\eps}$, with $I_\eps = [0,1] \cap \eps \Z$ and $d(x,z_q) \approx qd(x,\tau(x))$. Using this ruler and given $y$, we can approximate $\frac{d(x,y)}{d(x,\tau(x))}$ via

\[
    \inf\{n q : q \in I_\eps \cap [1/3, 2/3],  \exists y' \in X \text{ such that } \exists \text{ a path in H between $(z_q, x)$ and $(y',y)$ of length $n$}\}
\]
with the length of the path defined as the number of edges in it. The above optimization finds an approximately equally spaced sequence $x = x_0,\dots, x_n = y$ with $d(x_0, x_1) \approx d(z_q, x)$ and $r/3 \leq d(z_q, x) \leq 2r/3$. By the logic of the previous paragraph, we obtain $n \approx \frac{d(x,y)}{qd(x,\tau(x))}$, giving our desired $nq \approx \frac{d(x,y)}{d(x,\tau(x))}$. We then calculate $\frac{d(x,\tau(x))}{d(z,\tau(z))}$ for all distinct $x,z$ by dividing the learned ratios and this gives all distances up to an unknown constant, which can be resolved by the diameter 1 assumption.

\subsubsection{Observer Regions}
\label{sec:obsreg}

In the partial observability regime, we only have pairwise distances between points in $Y$ and $M$. Thus, our previous inner product framework is invalid. Our solution is to define $f_x:Y \to \R$ for all $x \in M$ to be $F(d(x,\cdot))$. If the distance between $x_1, x_2 \in M$ is small, then $\|f_x - f_y\|_2^2$ will also be small. In the case that $d(x_1, x_2)$ is large, (D3) ensures that there are enough points to witness the gap between them. Thus, we have a proximity sensor.

The situation for rulers is a bit more involved, because there isn't a natural algorithm to obtain the midpoint of two points in $M$ from our data. Instead we employ a novel splitting routine. Since $Y$ is a subset of $M$ and the clusters effectively form an $\eps$-net of $M$, we can assign clusters to points in $Y$. Thus we have approximate distances between points in $Y$. Suppose we take two far-away points in $Y$, labelled $z_0$ and $z_1$. Based on our data we can find a midpoint in $M$, labeled $z_{1/2}$. We then find a point $y_{1/2} \in Y$ that has the same distance to $z_0$ and obtain a midpoint of $z_0$ and $y_{1/2}$, termed $w_{1/4}$. From there we find the minimizer of $d(z_1,q)$ over $d(z_0,q) < d(z_0,w_{1/4})$ and label it $z_{1/4}$. Note that this minimizer forces $d(z_1, z_{1/4}) \approx d(z_0,z_1) - d(z_0,z_{1/4})$. We call this whole operation a \textbf{left split}. We also require another operation, the \textbf{right split}, which is just a left split with the roles of $z_0$ and $z_1$ reversed.

Observe that if $z_0$ and $z_1$ are at unit distance and we parameterize a geodesic between them by arc-length, the left and right splits correspond to the following operations on the parameter: $L:s \mapsto \frac{s}{2}$ and $R:s \mapsto \frac{1+s}{2}$. Suppose that we have a sequence of numbers $s_1, \dots, s_k$ with $s_1 = \frac{1}{2}$ and $s_{i+1} = L(s_i)$ or $R(s_i)$. Assume that all admissible $s_k$ in this setup form the set $T_k = \{\frac{q}{2^k}: 1 \leq q \leq 2^k \text{ odd}\}$ of odd dyadic rationals -- this holds for $s_1$ because $T_1 = \{\frac{1}{2}\}$. Then $L(T_k)$ is the set of odd dyadic rationals $\frac{q}{2^{k+1}}$ in $(0,\frac{1}{2})$ and $R(T_k)$ is the set of odd dyadic rationals in $(\frac{1}{2},1)$, thus the admissible $s_{k+1}$ are $T_{k+1}$. Thus, the union of admissible $s_j$ across $j \leq k$ gives the full set of dyadic rationals $D_k$. By this logic, we're able to obtain $\{z_s: s \in D_k\}$ satisfying $d(z_0, z_s) \approx s$, yielding the ruler $\widetilde{F}^{-1}(s) = \sup\{\ell : \ell \in D_k, F(d(z_0,z_\ell)) < s\}$ for distances from $Y$ to $M$. Further, the number of error-propagation steps required is logarithmic in the ruler's granularity (i.e. $k$ steps 
to obtain a set of ticks evenly spaced at $2^{-k}$-length intervals.)

In order to recover distances between points in $M$, note that there is a geodesic from some $y \in Y$ that passes through $x_1, x_2 \in M$, hence $d(x_1,x_2) =|d(y,x_1) - d(y,x_2)|$, so based on (D2) the following optimization approximates the distances in $M$.

\begin{equation}
    \sup_{y \in Y} \left|\tilde{d}(y,x_1) - \tilde{d}(y,x_2)\right|
\end{equation}

\subsection{Comparison with Previous Methods}
\label{section:comparison}

Some of our techniques overlap with those of the earlier ArXiv posting \cite{huang2025reconstructing}. That paper works in the latent space graph setting with the underlying space a manifold with curvature and volume bounds and with a known link function. In our initial ArXiv posting \cite{FMR2025}, we were unaware of the previous work \cite{huang2025reconstructing} and take this opportunity to correct the omission. Additionally, there has been a line of work since the two postings; see Section~\ref{subsec:announcement}.

Both \cite{noisyintrinsic} and \cite{huang2025reconstructing} use notions of either $L^2$ norms or inner products of distance functions to enable clustering, although they both work in the case of known $F$, and thus do not require a ruler construction. A comparison of sample complexities across the three existing methods for the full-data case is provided in \Cref{tab:comparison}. These comparisons are in the setting where $\M$ is a manifold, as this is where the previous methods apply. All entries are for calculating distances up to $\tO(\eps)$ error on an $\eps$-net of $\M$.

\begin{table}[h]
\centering
\renewcommand{\arraystretch}{1.3}
\begin{tabularx}{\textwidth}{X|X|X|X|X}
\toprule
\textbf{Paper} & \textbf{Samples} & \textbf{F} & \textbf{Noise} \\
\midrule
\cite{noisyintrinsic} & $\tO(\eps^{-2d})$ & Known, F=Id & i.i.d. additive, known variance \\
\midrule
\cite{huang2025reconstructing} & $\tO(\eps^{-d-2})$ & Known, decreasing bilipschitz & Bernoulli \\
\midrule
This paper & $\tO(\eps^{-2d-2})$  & Unknown, increasing bilipschitz & Subgaussian \\
\bottomrule
\end{tabularx}
\caption{Comparison of manifold distance recovery methods for recovering distances in an $O(\eps)$-net up to additive error $\tO(\eps)$.}
\label{tab:comparison}
\end{table}

Note that \cite{noisyintrinsic} provides a single sample-size bound for both cases and does not provide runtime or sample complexity bounds. For this reason we omit those figures from our three-way comparison. \cite{noisyintrinsic} works in the setting of $d'(x,y) = d(x,y) + N_{x,y}$ with all $N_{x,y} \sim N$ subexponential with mean $0$ and known variance $\sigma^2$, thus $F = I$. The authors make use of the fact that in the constant-variance setting $k(x,y) = \int (d(x,z)-d(y,z)^2d\mu(z) = \int [(d'(x,z) - d'(y,z))^2 - 2\sigma^2]d\mu(z)$, where the left integral, by the logic of \Cref{subsection:innerprodsproxsensors}, is a proximity sensor. In fact in the non-missing case it directly estimates distances. Therefore, they do not require any optimization formalism to construct clusters. For the missing data case, note that the conditions (C1)-(C2) directly generalize those on $p(x,y)$ in \cite{noisyintrinsic}: $c_1 p_1(d(x,y)) \leq p(x,y) \leq c_2 p_1(d(x,y))$ with $0 < c_1 < 1 < c_2$ and $p_1$ non-increasing with $p_1(0) = h > 0$ and $\|p_1\|_{C^1([0,1])} \leq H < \infty$. \cite{noisyintrinsic} estimates the number of points in $\eps$-balls to account for missing data. Since an $\eps$-ball has volume $\sim \eps^d$, a Hoeffding-type bound yields $N \gesim \eps^{-2d}$. This is distinct from the origin of the $\eps^{-2d}$ factor in our complexity, which follows from the quadratic discrepancy in the upper and lower bounds of \Cref{lem:sep}. Our extra factor of $2$ in the exponent follows from requiring $\sim \eps^{-2}$ many points in each $\eps^2$ ball. An analogous $\eps^{-d-2}$ term shows up in the sample complexity of \cite{noisyintrinsic}, but for $d \geq 2$ it is dominated by $\eps^{-2d}$.

\cite{huang2025reconstructing} works in the setting where $d'(x,y) \sim \text{Bernoulli}(F(d(x,y))$ denotes the existence of an edge in a random geometric graph, where $F$ is a known, decreasing bilipschitz function. The authors utilize the same notion of inner product as \Cref{eq:innerprodsproxsensors1}, but differ in its application to produce clusters. Although $F$ is decreasing there, the inner product logic, which only requires that $F$ is monotonic, is the same. They optimize over $(x,y) \in \M^2$ to find $x^* \in \M$ maximizing $\langle f_x, f_x\rangle$ and then use the fact that $\langle f_{x^*}, f_y\rangle = \frac{1}{2}(\|f_{x^*}\|_2^2 + \|f_y\|_2^2 - \|f_{x^*} - f_y\|_2^2)$ to find $y$ close to $x^*$. This technique yields the same quadratic gap in upper and lower bounds for the proximity sensors. However, once coarse clusters are calculated and expected noisy distances estimated, they are able to use the fact that $F$ is known to invert the distances, allowing them to utilize Riemannian geometry to calculate finer distances without inflating the sample complexity -- though they pay the same Hoeffding-type $+2$ in the exponent.

The most directly comparable result with \Cref{thm:mainresthm3} is \cite{fefferman2025reconstruction}, although it concerns a distinct setting and accounts for only deterministic error, not noise. In what follows, a $\delta$-approximate distance between points $x$ and $y$ in a metric space is a real number $\hat{d}(x,y)$ satisfying $|\hat{d}(x,y) -d(x,y)| < \delta$. \cite{fefferman2025reconstruction} recovers $O(\eps^{1/8})$-approximate distances between points in an $\eps^{1/2}$-net of a closed manifold $\M$ given access to the $\eps$-approximate distances between those points and an $\eps$-net of $U \subset \M$ satisfying $B(x_0,R) \subset U$ for at least one $x_0 \in U$ and $R$ larger than a manifold-determined constant. As a first point of comparison, the deterministic part of our algorithm (see \Cref{sec:partial-rulers}) takes in $\eps$-approximate distances between $\eps$-nets of both $\M$ and $Y$ and returns $\tO(\eps)$-approximate distances between the points in $\M$ (note the suppressed log factors). Now we show that these settings are distinct. Consider a polar cap on a sphere, which clearly satisfies the conditions of \cite{fefferman2025reconstruction}, and the equator. Observe that the great-circle going through any two distinct points on the equator never touches the cap, thus it fails (D2). On the other hand, if we have the ball as the observed region and its boundary as the observer, this satisfies (D1)-(D4) as explained in \Cref{sec:mainresults}, but does not contain an R-ball. Thus, the two cases are distinct.

\subsection{Outline of paper}
We give probabilistic and manifold preliminaries in \Cref{sec:prelim}. In \Cref{sec:alg1}, we present our algorithm for the full and missing data cases and then prove the observer-observed region result in \Cref{sec:partialobs}. We discuss some open questions and conclude in \Cref{sec:discussion}. Finally \Cref{sec:proofs} contains proofs of a few standard results utilized in \Cref{sec:alg1} and \Cref{sec:partialobs}.

\subsection{Announcement of Further Results}\label{subsec:announcement}

Since our initial ArXiv posting \cite{FMR2025}, there have been several further postings related to this problem. We announce them here.

One line of work has focused on the following problem: given access to noisy distances whose expectations are a bilipschitz function of the true distance and whose distributions are bounded in subgaussian norm (the setting of this paper) between an $\delta$-net of a d-manifold $\M$, for what choices of $\eps_1, \eps_2 \le \delta$ is it possible to recover the pairwise distances between points in an $\eps_1$-net up to error $\eps_2$? \cite{huang2026denoising} achieves $\eps_1 = \delta$ and $\eps_2 = O(\delta^{\frac{2d}{d+5}})$ in the case of a $C^2$ manifold. In a forthcoming paper, we assume that both the manifold and bilipschitz function $F$ are $C^m$ for $m \leq d$ and achieve $\eps_1 = \eps_2 = \tO(\delta^{\frac{dm}{d+m}})$. Further, in that paper we prove that for the round sphere $\mathbb{S}^d$, it is impossible to recover distances to accuracy better than $\tO(\delta^{n/2})$, showing that the bound for $m = n$ is tight.

Another line \cite{huang2026denoisingmetric} continues the work of extending the results proved in this paper and \cite{huang2025reconstructing} to general metric measure spaces, developing an algorithm to recover local clusters in metric measure spaces where all balls of radius $r \in I$ carry measure at least $\phi(r) \ge 0$ with $I$ either $[0,r_*)$ for $r_* > 0$ or $(0,\infty)$ and $\phi$ non-decreasing, which allow for the approximate recovery of $F(d(x,y))$ for all points (in the known $F$ case this gives the distances.)


\textbf{Acknowledgments.} C.F. and J.M. are partially supported by NSF Grant DMS-1608782 and AFOSR Grant FA9550-23-1-0273. K.R. is supported by a NSF GRFP fellowship, a Simons Foundation Dissertation Fellowship in Mathematics, and a Cubist/Point72 PhD Fellowship. We thank Matti Lassas, Hariharan Narayanan, Hau Tieng Wu, and Aalok Gangopadhyay for helpful discussions.

\section{Preliminaries}\label{sec:prelim}
\subsection{Probability}
\begin{definition}[Orlicz norm]
    If $X$ is a sub-Gaussian random variable, we define its Orlicz norm $\| X \|_{\psi_2}$ by
    \[
\norm{X}_{\psi_2} := \inf \{ c \ge 0 : \E\left( e^{(X - \E X)^2/c^2} \right) \le 2 \}.
\]
\end{definition}
\begin{theorem}[Sub-Gaussian Hoeffding Inequality]\label{thm:hoeffding} \cite[Theorem 2.6.2]{vershynin}
Let $X_1, \dots, X_n$ be independent sub-Gaussian random variables with means $\mathbb{E}[X_i]$ and Orlicz norms $\| X_i \|_{\psi_2}$.
Then for all $t \ge 0$,
\[
\Pr\!\left(
\left| \sum_{i=1}^n (X_i - \mathbb{E}[X_i]) \right| \ge t
\right)
\le
2 \exp\!\left(
-\frac{ct^2}{2 \sum_{i=1}^n \|X_i\|_{\psi_2}^2}
\right),
\]
where $c$ is a universal constant.
\end{theorem}

\begin{theorem}[Hoeffding Inequality]\label{thm:hoeffding 2} \cite[Theorem 2.2.6]{vershynin}
Let $X_1, \dots, X_n$ be independent random variables such that $|X_i| \le C$ almost surely.
Then for all $t \ge 0$,
\[
\Pr\!\left(
\left| \sum_{i=1}^n (X_i - \mathbb{E}[X_i]) \right| \ge t
\right)
\le
2 \exp\!\left(
-\frac{t^2}{8nC^2}
\right).
\]
\end{theorem}

\begin{theorem}[Bernstein]\label{lem:bernstein} \cite{bernstein1924}
Let $X_1, \dots, X_n$ be independent random variables such that $|X_i - \E X_i| \le C$ almost surely.
Then for all $t \ge 0$,
\[
\Pr\!\left(
\left| \sum_{i=1}^n (X_i - \mathbb{E}[X_i]) \right| \ge t
\right)
\le
2 \exp\!\left(
-\frac{t^2}{2\sum_{i=1}^n \Var(X_i) + \frac{2}{3} Ct}
\right).
\]
\end{theorem}

\subsection{Sampling Bounds}
The following is a standard Chernoff bound.
\begin{proposition}[Chernoff Bound]\label{prop:chernoff}
    Let $\M$ be a measure metric space and $\mu$ be a probability measure on $\M$. Let $X$ be an finite i.i.d sample from $\M$ according to $\mu$, and let $\Lambda$ be a $\mu$-measurable set. Then
    \begin{equation*}
\mathbb{P}\Big( |\Lambda \cap X| < \frac{1}{2} \mu(\Lambda) |X| \Big) \le e^{-\frac{1}{8} \mu(\Lambda) |X|}.
    \end{equation*}
\end{proposition}
For completeness, we include proof sketches for the following propositions in \Cref{sec:proofs}.
\begin{proposition}\label{prop:density}
Let $\M$ be a measure metric space and $\mu$ be a probability measure on $\M$. For $r > 0$, let $\VR(r)$ be the infimal possible measure of an $r$-ball on $\M$. Let $Y$ be a finite i.i.d sample from $\mu$.
With probability $\ge 1 - \frac{1}{\VR(\eps/4)} e^{-\frac{1}{8} \VR(\frac{\eps}{2}) |Y|})$, the set $Y$ is $\eps$-dense in $\M$.
\end{proposition}

\section{Single Region}\label{sec:alg1}
In this section, we prove the full and missing data cases, Theorems \ref{thm:mainresthm1} and \ref{thm:mainresthm2}, where the sampled distances are between points in the same region. We begin with a brief sketch of the algorithm.
\begin{enumerate}
    \item Find a cluster $\C(y)$ around each $y \in Y$ such that
\[
B(y, \varepsilon') \subseteq \C(y) \subseteq B(y, \varepsilon),
\]
for some $\varepsilon' < c\varepsilon$ to be chosen later.

    \item Compute ``distance proxy'' numbers $A(x, y)$ such that for all $x\in X$ and $y,z \in Y$, if $d(x, y) \ge d(x, z) + 17\eps$, then $A(x, y) > A(x, z)$.

    \item Recover all $d(y, z)$ with $y, z \in Y$ up to $O(\eps \log \eps^{-1})$.

    \item For each $x \in X$, choose a representative $\pi(x) \in Y$ minimizing the proxy $A(x,\cdot)$. Since $Y$ is $\eps$-dense, this gives $d(x,\pi(x))=O(\eps)$, and we output $\widehat d(x,x') := \widehat d_Y(\pi(x),\pi(x'))$ for all $x,x' \in X$. This adds only $O(\eps)$ to the additive error.
\end{enumerate}

Step 1 will be motivated in \Cref{sec:inner prod} and proved in \Cref{sec:cluster construction}. Step 2 will also be proved in \Cref{sec:cluster construction}, and Step 3 will be completed in \Cref{section:findingdistancesfromclusters}.

\subsection{Obtaining Noisy Expected Distances}
We prove the following theorem, which is about an abstract metric space satisfying certain properties that random samples of manifolds have (thanks to \Cref{prop:kappa exists} and \Cref{prop:density}).
\begin{theorem}\label{thm:main_technical}
    Let $(X, d)$ be a metric space with $|X| = N$ and $Y$ denote the first $N_0$ elements of $X$. For $x, y \in X$, let $p(x, y) \in [0, 1]$ denote the probability that observation $d'(x, y)$ is present, and denote $m(x, y)$ to equal $1$ if observation $d'(x, y)$ is present, and $0$ otherwise. Suppose that $X$ satisfies the following a priori properties.
    \begin{itemize}
        \item every ball $B(x, \delta)$ with $x \in X$ contains at least $\alpha N$ many elements of $X\setminus Y$;

        \item if $x, y \in X$ and $d(x, y) \le r_0$, then $p(x, y) \ge \phi$;

        \item if $x, y \in X$ are distinct and satisfy  $d(x, y) \le \delta$, then $\frac{1}{|Y|} \sum_{z \in Y} p(x, z) p(y, z) > 2c_1$;
    
        \item if $d(x, y) > 4\eps$, $d(x, v) \le \delta$, $d(y, w) \le \delta$, and $\frac{1}{|Y|} \sum_{z \in Y} p(x, z) p(y, z) > c_1$, the set $\Lambda_{x,y,v,w} := \{ z : |d(x, z) -  d(y, z)| \ge \eps, p(x, z) p(y, z) p(v, z) p(w, z) \ge c_2 \}$ satisfies $|\Lambda_{x, y, v,w} \cap Y| \ge \kappa |Y|$ for each distinct $x, y, v, w \in X$;
    \end{itemize}
    Suppose we have a noisy distance function $d'(x, y)$ satisfying the following properties:
\begin{itemize}
    \item (Identity) $d'(x, x) = 0$;
    
    \item (Symmetry) $d'(x, y) = d'(y, x)$;

    \item (Subgaussian) $d'(x, y)$ is sub-Gaussian with Orlicz norm $\norm{d'(x, y)}_{\psi_2} \le C_1$;

    \item (``Independence'') $\{ d'(x, y) \}_{\{x, y\} \subset X}$ is a set of random variables such that for every $x \in X$, the variables $\{ d'(x, y) \}_{y \in X}$ are independent conditioned on any subset of the other variables $\{ d'(y, z) \}_{y, z \neq x}$.
\end{itemize}
and furthermore whose expectation $f(x, y) = \E d'(x, y)$ satisfies the following properties:
\begin{itemize}
    \item (Bounded expectation) $|f(x, y)| \le C_2$ for all $x, y \in X$;
    
    \item (Lower bound) $f(x, y) - f(x, z) \ge \sigma$ for all distinct $x, y, z \in X$ with $d(x, y) \ge d(x, z) + \eps$.
    
    \item (Uniform continuity) If $y, z \in X$ satisfy $d(y, z) \le \delta$, then $|f(x, y) - f(x, z)| \le \frac{\kappa \sigma^2 c_2}{8C_2}$ for all $x \neq y, z$.
\end{itemize}
Suppose that $N_0 \ge C \frac{C_1^4 + C_2^4}{\kappa^2 \sigma^4 c_2^2}\log \frac{N}{\theta} + \frac{C}{c_1^2} \log \frac{N}{\theta}$ and $\alpha N \ge \frac{CC_1^2}{\phi \sigma^2} \log \frac{N_0}{\theta}$, where $C > 0$ is an absolute constant. Then with probability $\ge 1 - \theta$, we can compute numbers $A(x, y) \in \R \cup \{ +\infty \}$ for all $x\in X$ and $y\in Y$ in time $N^3 N_0^2$ such that:
\begin{enumerate}
    \item for every $x \in X$ and $y \in Y$, if $d(x, y) \le r_0 - 8\eps$, then $A(x, y) \neq \infty$;
    
    \item for every $x \in X$ and $y, z \in Y$ satisfying $d(x, y) \ge d(x, z) + 17\eps$ and $A(x, y), A(x, z) \neq +\infty$, we have $A(x, y) > A(x, z)$.
\end{enumerate}
If all $p(x, y) = 1$ and $c_2 = 1$ (the non-missing case), we can in fact compute numbers $A(x, y) \in \R$ for all $x\in X$ and $y\in Y$ such that for all $x\in X$ and $y,z\in Y$ with $d(x, y) \ge d(x, z) + 17\eps$, then $A(x, y) > A(x, z)$ in time $N^2 N_0$. We also obtain slightly relaxed guarantees $N_0 \ge C \frac{C_1^4 + C_1^2 C_2^2+C_2^4}{\kappa^2 \sigma^4}\log \frac{N}{\theta}$ and $\alpha N \ge \frac{CC_1^2}{\sigma^2} \log \frac{N_0}{\theta}$.
\end{theorem}

\begin{remark}
    Think of $\sigma \sim \eps$. It turns out that combining with \Cref{prop:density}, we need to take $N_0 \sim \eps^{-\max \{ d, 4 \}} (d \log \frac{1}{\eps} + \log \frac{1}{\theta})$, $\alpha \sim \eps^{2d}$, and $N \sim \eps^{-2d-2} (\log \eps^{-1} + \log \theta^{-1})$.
\end{remark}

\subsubsection{Inner Products}\label{sec:inner prod}
\begin{definition}
Fix a subset $Y \subseteq X$. For functions $f, g : Y \to \R$, we define
\[
\langle f, g \rangle = \frac{1}{|Y|} \sum_{y \in Y} f(y) g(y).
\]
\end{definition}

\begin{remark}
    We can also define the averaging to be over a set $Y'$ which can be different from $Y$; this can be useful for further technical applications.
\end{remark}
The motivation for considering inner products is that they can be estimated effectively via the following concentration statement for inner products. The statement and proof technique may be of independent interest. The statement is not the most general possible in order to give a clean exposition; the reader is encouraged to relax the constraints of the proposition. The proof is deferred to \Cref{sec:proofs}.
\begin{proposition}[Sub-Gaussian Hoeffding Inequality for Inner Products]\label{prop:hoeffding ip}
Let $X_1, \dots, X_n$, $Y_1, \dots, Y_n$ be sub-Gaussian random variables with means 
$\mathbb{E}[X_i]$, $\mathbb{E}[Y_i]$.

Suppose that:
\begin{itemize}
    \item $Y_1, \dots, Y_n$ are mutually independent with Orlicz norms $\|Y_i\|_{\psi_2} \le K$;
    \item Conditioned on $Y_1, \dots, Y_n$, the $X_1, \dots, X_n$ are mutually independent with Orlicz norms $\le K$
    \item $\E[X_i \mid Y_1, \cdots, Y_n] = \E[X_i]$, $|\E[X_i]| \le L$, and $|\E[Y_i]| \le L$ for all $1 \le i \le n$.
\end{itemize}
(These assumptions hold if $X_1, \cdots, X_n, Y_1, \cdots, Y_n$ are all mutually independent.)
Then for all $t \ge 0$,
\[
\Pr\!\left(
\left| \sum_{i=1}^n (X_i Y_i - \mathbb{E}[X_i]\mathbb{E}[Y_i]) \right| \ge t
\right)
\le
5 \exp\!\left(
-\frac{ct^2}{16K^2 ((K^2 + L^2) n + t)}
\right),
\]
where $c$ is a universal constant.
\end{proposition}

\begin{remark}
    If $X_1, \cdots, X_n, Y_1, \cdots, Y_n$ are mutually independent, then each $X_i Y_i$ is sub-exponential, hence we can use the sub-exponential Bernstein inequality. However, our approach works for our weaker mutual dependencies, at the cost of worse constants.
\end{remark}

In our context, we have a noisy collection of vectors $\{ f_x \}_{x \in X}$ defined by $f_x (y) = d'(x, y)$ and by \Cref{prop:hoeffding ip} we can efficiently estimate inner products $\langle f_x, f_y \rangle$ with $x \neq y$; unfortunately, this is not enough to accurately determine norms $\norm{f_x - f_y}^2$, which is the target quantity for determining which pairs $(f_x, f_y)$ are close. Instead, we resort to estimating the norm via the following lemma.
\begin{lemma}[Separation Lemma]\label{lem:sep}
Fix $L, \delta > 0$, and suppose that for every $x \in X$, we have a function $f_x : Y \to \R$. Suppose that:
\begin{itemize}
    \item (Bounded) Every $\norm{f_x}_2 \le L$;

    \item (Dense) For any $x \in X$, there exists $y \neq x \in X$ with $\norm{f_x - f_y}_2 \le \delta$.
\end{itemize}
Then for any $x \neq y$,
\[
\frac{1}{2} \norm{f_x - f_y}_2 (\norm{f_x - f_y}_2 - 2\delta) \le \sup_{z \in X, z \neq x, y} |\langle f_x - f_y, f_z \rangle| \le L \norm{f_x - f_y}_2.
\]
\end{lemma}

\begin{proof}
We first look at the left inequality. If $\norm{f_x - f_y}_2 \le 2\delta$ then it is trivial, so assume $\norm{f_x - f_y}_2 > 2\delta$. Note that one of $|\left< f_x, f_x - f_y \right>| \ge \frac{1}{2} \norm{f_x - f_y}^2$ or $|\left< f_y, f_x - f_y \right>| \ge \frac{1}{2} \norm{f_x - f_y}^2$; otherwise,
\[
    \norm{f_x - f_y}^2 = \left< f_x, f_x - f_y \right> - \left< f_y, f_x - f_y \right> < \frac{1}{2} \norm{f_x - f_y}^2 + \frac{1}{2} \norm{f_x - f_y}^2 = \norm{f_x - f_y}^2.
\]
Without loss of generality, assume $|\left< f_x, f_x - f_y \right>| \ge \frac{1}{2} \norm{f_x - f_y}^2$. Using the dense assumption, pick $z \neq x, y$ such that $\norm{f_x - f_z}_2 \le \delta$. Then by Cauchy-Schwarz,
\begin{equation*}
    |\langle f_x - f_y, f_z \rangle| \ge |\langle f_x - f_y, f_x \rangle| - |\langle f_x - f_y, f_x - f_z \rangle| \ge \frac{1}{2} \norm{f_x - f_y}^2_2 - \delta \norm{f_x - f_y}_2.
\end{equation*}
This proves the first claim.

For the right inequality, we use Cauchy-Schwarz and our boundedness assumption to get, for all $z, w \in X$,
\[
|\langle f_x - f_y, f_z \rangle| \le \norm{f_x - f_y}_2 \norm{f_z}_2 \le L \norm{f_x - f_y}_2. \qedhere
\]
\end{proof}

\begin{remark}
If our bounded assumption is $\norm{f_x}_\infty \le L$, then our result can be improved to $\le L \norm{f_x - f_y}_1$.
\end{remark}

\subsubsection{Cluster Construction}\label{sec:cluster construction}
Now, we begin the proof of \Cref{thm:main_technical} in the non-missing data case, which means we will assume $p(x, y) = 1$ and $c_2 = 1$. For convenience, define $d'(x, x) = 0$ and $f(x, x) = 0$ for all $x \in X$.
\begin{definition}\label{def:cluster}
Fix a subset $Y \subset X$. We define
\[
\C(x) = \{ y \in X : \sup_{z \in X, z \neq x, y} \left| \frac{1}{|Y|} \sum_{v \in Y} (d'(x, v) - d'(y, v))d'(z, v) \right| \le \frac{3}{16} \kappa \sigma^2 \}.
\]
Note that $x \in \C(x)$.
\end{definition}

\begin{proposition}\label{prop:cluster}
    With probability at least $1 - 5|X|^2 \exp(-c|Y| \kappa^2\sigma^4/(25600C_1^2 (C_1^2 + C_2^2 + 1)))$, we have that for all $x \in Y$,
    \begin{equation}\label{eqn:cluster}
        B(x, \delta) \subset \C(x) \subset B(x, 4\eps).
    \end{equation}
    After a precomputation cost of $|X|^2 |Y|$, each cluster $\C(x)$ can be computed in $|X|^2$ time.
\end{proposition}

\begin{proof}
    If we precompute every $\sum_{v \in Y} d'(x, v) d'(y, v)$ for each pair $x, y \in X$, each computation of $\frac{1}{|Y|} \sum_{v \in Y} (d'(x, v) - d'(y, v))(d'(z, v) - d'(w, v))$ can be done in $O(1)$ time. For each $y \in X$, it takes $O(|X|)$ time to compute the supremum. Now we turn to proving \eqref{eqn:cluster}.

    Note that by independence of $d'(x, v)$ and $d'(z, v)$,
    \begin{equation*}
        \E \left[ \frac{1}{|Y|} \sum_{v \in Y} d'(x, v) d'(z, v) \right] = \frac{1}{|Y|} \sum_{v \in Y} f(x, v) f(z, v) = \left< f_x, f_z \right>,
    \end{equation*}
    where $f_x : Y \to \R$ is defined by $f_x (y) = f(x, y)$.
    Thus, by \Cref{prop:hoeffding ip}, with probability at least $1 - 5|X|^2 \exp(-c|Y| \kappa^2 \sigma^4/(25600C_1^2 (C_1^2 + C_2^2 + 1)))$, we have that for all distinct $x \neq z \in X$,
    \begin{equation}\label{eqn:conc for d' 1}
        \left| \frac{1}{|Y|} \sum_{v \in Y} d'(x, v) d'(z, v) - \left< f_x, f_z \right> \right| \le \frac{\kappa \sigma^2}{40}.
    \end{equation}
    By using triangle inequality on a few applications of \eqref{eqn:conc for d' 1}, we have for every distinct $x, y, z \in X$,
    \begin{equation}\label{eqn:conc for d'}
        \left| \frac{1}{|Y|} \sum_{v \in Y} (d'(x, v) - d'(y, v))d'(z, v) - \left< f_x - f_y, f_z \right> \right| \le \frac{\kappa \sigma^2}{20}.
    \end{equation}
    Observe that $\norm{f_x}_2 = \left(\frac{1}{|Y|}\sum_{y \in Y} |f(x, y)|^2\right)^{1/2} \le C_2$ and for every $x \in X$, there exists $y \neq x \in X$ such that $d(x, y) \le \delta$, hence by the uniform continuity assumption in Theorem \ref{thm:main_technical} (note that the assumption requires $z \neq x, y$),
    \begin{equation}\label{eqn:upper bound x, y close}
        \norm{f_x - f_y}_2^2
        \le \frac{1}{|Y|}\sum_{z\in Y\setminus\{x,y\}} |f(x,z)-f(y,z)|^2
        +\frac{2C_2^2}{|Y|}
        \le \left(\frac{\kappa\sigma^2}{8C_2}\right)^2+\frac{2C_2^2}{|Y|} \le \left(\frac{\kappa\sigma^2}{8C_2}+\frac{2C_2}{\sqrt{|Y|}}\right)^2.
    \end{equation}
    If $y \in B(x, \delta)$, then $\norm{f_x - f_y} \le \frac{\kappa \sigma^2}{8C_2} + \frac{2C_2}{\sqrt{|Y|}}$ using the same calculation as in \eqref{eqn:upper bound x, y close}, so by Lemma \ref{lem:sep} and \eqref{eqn:conc for d'} (and using that $N_0 = |Y|$ is sufficiently large),
    \begin{equation*}
        \sup_{z \in X, z \neq x, y} \left| \frac{1}{|Y|} \sum_{v \in Y} (d'(x, v) - d'(y, v))d'(z, v) \right| \le C_2 \cdot \left( \frac{\kappa \sigma^2}{8C_2} + \frac{2C_2}{\sqrt{|Y|}} \right) + \frac{1}{20} \kappa \sigma^2 \le \frac{3}{16} \kappa \sigma^2.
    \end{equation*}
    If $d(x, y) \ge 4\eps$, then by assumption we have $2C_2 \ge \norm{f_x - f_y}$ and
    $\norm{f_x - f_y}^2 \ge \kappa \sigma^2$, hence by Lemma \ref{lem:sep}, we have
    \begin{equation*}
        \sup_{z \in X, z \neq x, y} \left| \frac{1}{|Y|} \sum_{v \in Y} (d'(x, v) - d'(y, v))d'(z, v) \right| \ge \frac{1}{2} \kappa \sigma^2 - 2C_2 \cdot \frac{\kappa \sigma^2}{8C_2} - \frac{1}{20} \kappa \sigma^2 = \frac{1}{5} \kappa \sigma^2 > \frac{3}{16}\kappa\sigma^2.
    \end{equation*}
\end{proof}

\subsubsection{Average Noisy Distances Oracles}
Note that the previous subsections on cluster construction have only conditioned on noisy distances $d'(x, y)$ with $\{ x, y \} \cap Y \neq \emptyset$; in particular, we have not conditioned on $d'(x, y)$ with $\{ x, y \} \cap Y = \emptyset$.
We truncate each $\C(y)$, $y\in Y$, so that $\C(y)\cap Y=\emptyset$. For $x\in X$, define an anchor
\[
    a(x)=\begin{cases}
        x, & x\notin Y,\\
        x'(x), & x\in Y,
    \end{cases}
\]
where $x'(x)$ is an arbitrary point of $\C(x)$ when $x\in Y$; the first a priori assumption guarantees that such a held-out anchor exists. Hence $d(a(x),x)=0$ for $x\notin Y$ and $d(a(x),x)\le 4\eps$ for $x\in Y$. We compute the noisy averages (for $x\in X$ and $y\in Y$)
$$
A(x, y) := \frac{1}{|\C(y)|} \sum_{y' \in \C(y)} d'(a(x), y').
$$
Note that
$$
\E A(x, y) = \frac{1}{|\C(y)|} \sum_{y' \in \C(y)} \E d'(a(x), y').
$$

\begin{lemma}\label{assump:1}
    $\E A(x, y) - \E A(x, z) \ge \sigma$ for all $x\in X$ and $y,z \in Y$ with $d(x, y) \ge d(x, z) + 17\eps$.
\end{lemma}

\begin{proof}
    If $d(x, y) \ge d(x, z) + 17\eps$, then for every $y' \in \C(y)$ and $z' \in \C(z)$, the inclusions $\C(y),\C(z)\subset B(\cdot,4\eps)$ give
    \[
        d(a(x),y') \ge d(x,y)-d(a(x),x)-4\eps,
        \qquad
        d(a(x),z') \le d(x,z)+d(a(x),x)+4\eps.
    \]
    Since $d(a(x),x)\le4\eps$ when $x\in Y$ and $d(a(x),x)=0$ otherwise, we obtain $d(a(x),y')\ge d(a(x),z')+\eps$. Thus by the fact that an average of a sequence is between its minimum and maximum element and using our theorem assumption, we have
    \[
    \E A(x, y) - \E A(x, z) \ge \min_{y' \in \C(y)} f(a(x), y') - \max_{z' \in \C(z)} f(a(x), z') \ge \sigma. \qedhere
    \]
\end{proof}

Now by Theorem \ref{thm:hoeffding}, with probability $1 - 2|X||Y| \exp(-c\alpha N \sigma^2/C_1^2)$, we have for all $x\in X$ and $y \in Y$,
\begin{gather*}
    \left| A(x, y) - \E A(x, y) \right| \le \frac{\sigma}{3}.
\end{gather*}
Thus, if $x\in X$, $y,z\in Y$, and $d(x, y) \ge d(x, z) + 17\eps$, we get $A(x, y) > A(x, z)$, as desired. The additional values $A(x,y)$ with $x\notin Y$ require only $O(NN_0\max_y |\C(y)|)\le O(N^2N_0)$ arithmetic operations, which is the same order as the precomputation.

Finally, to make the probability of success at least $1-\theta$, it suffices to ensure that:
\begin{gather*}
    5|X|^2 \exp(-c|Y| \kappa^2 \sigma^4 / (C_1^4 + C_1^2 C_2^2)) \le \frac{\theta}{2}, \\
    2|X||Y| \exp(-c\alpha N \cdot \sigma^2/C_1^2) \le \frac{\theta}{2}.
\end{gather*}
One can check that the choices of $N_0$ and $N$ in the theorem statement work.

\subsubsection{Handling Missing Data}\label{sec:missing noise}
The proof of the missing case follows similarly to the non-missing case, but with some additional technical complications, so we will start from scratch.
Using \Cref{prop:hoeffding ip} applied to the random variables $d'(x, z) m(x, z) m(y, z) m(v, z) m(w, z)$ (which has Orlicz norm $\le O(C_1 + C_2)$, since $d'$ is subgaussian with mean $O(C_2)$ and $m(\cdot, \cdot)$ is bounded in $[0, 1]$, hence subgaussian) and $d'(v, z)$, with probability $\ge 1 - 5|X|^4 \exp(-c|Y| \kappa^2 \sigma^4 c_2^2 / (C_1^4 + C_2^4))$, we can guarantee for all $x, y, v, w \in X$,
\begin{multline}\label{eqn:hoeff0}
\bigg| \frac{1}{|Y|} \sum_{z \in Y} \big(d'(x, z) d'(v, z) m(x, z) m(y, z) m(v, z) m(w, z) \\
- f(x, z) f(v, z) p(x, z) p(y, z) p(v, z) p(w, z) \big) \bigg| < \frac{1}{40} \kappa \sigma^2 c_2,
\end{multline}
which implies
\begin{multline}\label{eqn:hoeff1}
\bigg| \frac{1}{|Y|} \sum_{z \in Y} \big( (d'(x, z) - d'(y, z))(d'(v, z) - d'(w, z)) m(x, z) m(y, z) m(v, z) m(w, z) \\
- (f(x, z) - f(y, z))(f(v, z) - f(w, z)) p(x, z) p(y, z) p(v, z) p(w, z) \big) \bigg| < \frac{1}{10} \kappa \sigma^2 c_2,
\end{multline}
Using \Cref{thm:hoeffding 2}, with probability $\ge 1-2|X|^2 \exp(-|Y| c_1^2/32)$, we have for all $x, y \in X$,
\begin{equation}\label{eqn:hoeff2}
\bigg| \frac{1}{|Y|} \sum_{z \in Y} \big( m(x, z) m(y, z) - p(x, z) p(y, z) \big) \bigg| < 0.5c_1.
\end{equation}

Let
\begin{multline*}
\C(x) = \{ y \in X : \frac{1}{|Y|} \sum_{z \in Y} m(x, z) m(y, z) \ge 1.5c_1 \text{ and}\\
\sup_{\substack{v, w \in X\\|\{ v, w, x, y\}| = 4}} \left| \frac{1}{|Y|} \sum_{z \in Y} (d'(x, z) - d'(y, z))(d'(v, z) - d'(w, z)) m(x, z) m(y, z) m(v, z) m(w, z) \right| \le \frac{3}{8} \kappa \sigma^2 c_2 \}.
\end{multline*}
Computing each cluster takes time $N^3 N_0$, so it takes total time $N^3 N_0^2$ to compute $\C(x)$ for each $x \in Y$. It turns out that this runtime dominates the algorithm. Now, we show that $B(x, \delta) \subset \C(x) \subset B(x, 4\eps)$.

Suppose $d(x, y) < \delta$. Then by \Cref{thm:main_technical} assumption, we have $\frac{1}{|Y|} \sum_{z \in Y} p(x, z) p(y, z) \ge 2c_1$, so by \eqref{eqn:hoeff2}, we have $\frac{1}{|Y|} \sum_{z \in Y} m(x, z) m(y, z) \ge 1.5c_1$. Furthermore, for all $v, w \in X$, we have
$$
\left| \frac{1}{|Y|} \sum_{z \in Y} (f(x, z) - f(y, z))(f(v, z) - f(w, z)) p(x, z) p(y, z) p(v, z) p(w, z) \right| \le 2C_2 \cdot \frac{\kappa \sigma^2 c_2}{8C_2} + \frac{2C_2^2}{|Y|} \le 0.26 \kappa \sigma^2 c_2.
$$
Combining this with \eqref{eqn:hoeff1} shows that $y \in \C(x)$.
Now suppose $d(x, y) > 4\eps$. Choose $v \in B(x, \delta)$ and $w \in B(y, \delta)$ such that $x, y, v, w$ are distinct (this is doable by the first assumption in \Cref{thm:main_technical}). Then either $\frac{1}{|Y|} \sum_{z \in Y} m(x, z) m(y, z) < 1.5c_1$, in which case $y \not\in \C(x)$, or $\frac{1}{|Y|} \sum_{z \in Y} m(x, z) m(y, z) \ge 1.5c_1$, in which case $\frac{1}{|Y|} \sum_{z \in Y} p(x, z) p(y, z) \ge c_1$ by \eqref{eqn:hoeff2} and so by \Cref{thm:main_technical} assumption, we get
\begin{align*}
&\left| \frac{1}{|Y|} \sum_{z \in Y} (f(x, z) - f(y, z))(f(v, z) - f(w, z)) p(x, z) p(y, z) p(v, z) p(w, z) \right| \\
&\ge \left| \frac{1}{|Y|} \sum_{z \in Y} |f(x, z) - f(y, z)|^2 p(x, z) p(y, z) p(v, z) p(w, z) \right| \\
&- \left| \frac{1}{|Y|} \sum_{z \in Y} (f(x, z) - f(y, z))(f(x, z) - f(v, z)) p(x, z) p(y, z) p(v, z) p(w, z) \right| \\
&- \left| \frac{1}{|Y|} \sum_{z \in Y} (f(x, z) - f(y, z))(f(y, z) - f(w, z)) p(x, z) p(y, z) p(v, z) p(w, z) \right| \\
&\ge \kappa c_2 \sigma^2 - 4C_2 \cdot \frac{\kappa \sigma^2 c_2}{8C_2} - \frac{4C_2^2}{|Y|} \ge 0.49 \kappa c_2 \sigma^2.
\end{align*}
Combining this with \eqref{eqn:hoeff1} and \eqref{eqn:hoeff2} shows that $y \not\in \C(x)$.

Recall that we have conditioned on $d'(x, y)$ with $\{ x, y \} \cap Y \neq \emptyset$; in particular, we have not conditioned on $d'(x, y)$ with $\{ x, y \} \cap Y = \emptyset$. Truncate the clusters so that $\C(y)\cap Y=\emptyset$ for every $y\in Y$. For $x\in X$, define the anchor $a(x)=x$ if $x\notin Y$ and otherwise choose $a(x)=x'(x)\in\C(x)$; the first a priori assumption guarantees that such a held-out anchor exists. Then $d(a(x),x)=0$ for $x\notin Y$ and $d(a(x),x)\le4\eps$ for $x\in Y$. We compute the noisy averages (for $x\in X$ and $y\in Y$)
$$
A(x, y) := \begin{cases}
    \frac{\sum_{y' \in \C(y)} d'(a(x), y') m(a(x), y')}{\sum_{y' \in \C(y)} m(a(x), y')}, & \text{ if } \sum_{y' \in \C(y)} m(a(x), y') > \frac{1}{2} \alpha N \phi; \\
    +\infty & \text{ otherwise.}
\end{cases}
$$
(Note that $A$ is not symmetric.) Intuitively, $A(x, y)$ is close to $f(x, y)$, but only if the denominator is sufficiently large.

By Theorem \ref{thm:hoeffding}, we have with probability $\ge 1 - 2|X||Y| \exp(-\frac{c}{2} \alpha N \phi \cdot \sigma^2/C_1^2)$,
\begin{equation}\label{eqn:bound on A}
    A(x, y) \neq \infty \implies \left| A(x, y) - \frac{\sum_{y' \in \C(y)} f(a(x), y') m(a(x), y')}{\sum_{y' \in \C(y)} m(a(x), y')} \right| \le \frac{\sigma}{3}.
\end{equation}
Now we verify the conclusion of \Cref{thm:main_technical}.
First, if $x\in X$, $y\in Y$, and $d(x, y) \le r_0 - 8\eps$, then $d(a(x), y') \le r_0$ for any $y' \in \C(y)$, so $p(a(x), y') \ge \phi$ by \Cref{thm:main_technical} assumption. Hence, $\sum_{y' \in \C(y)} p(a(x), y') \ge \alpha N \phi$, so by applying \Cref{lem:bernstein} to the Bernoulli independent random variables $m(a(x), y')$, whose variance is $p(a(x), y') (1-p(a(x), y')) \le p(a(x), y')$, we get
\begin{equation*}
    \Pr\!\left( \sum_{y' \in \C(y)} m(a(x), y') < \frac{1}{2} \sum_{y' \in \C(y)} p(a(x), y') \right) \le 2 \exp\!\left(
-\frac{3}{32} \sum_{y' \in \C(y)} p(a(x), y')
\right).
\end{equation*}
Thus, we get $\sum_{y' \in \C(y)} m(a(x), y') \ge \frac{1}{2} \alpha N \phi$, which implies $A(x, y) \neq +\infty$, with probability $\ge 1 - 2\exp(-\frac{3}{32} \alpha N \phi)$.

Next, we claim if $x\in X$ and $y,z \in Y$ satisfy $d(x, y) \ge d(x, z) + 17\eps$ and $A(x, y), A(x, z) \neq +\infty$, then $A(x, y) > A(x, z)$. Indeed, as in \Cref{assump:1}, every $y'\in\C(y)$ and $z'\in\C(z)$ satisfy $d(a(x),y')\ge d(a(x),z')+\eps$. Thus, by \eqref{eqn:bound on A},
$$
\frac{\sum_{y' \in \C(y)} f(a(x), y') m(a(x), y')}{\sum_{y' \in \C(y)} m(a(x), y')} - \frac{\sum_{z' \in \C(z)} f(a(x), z') m(a(x), z')}{\sum_{z' \in \C(z)} m(a(x), z')} \ge \min_{y' \in \C(y)} f(a(x), y') - \max_{z' \in \C(z)} f(a(x), z') \ge \sigma.
$$
Finally, to make the probability of success at least $1-\theta$, it suffices to ensure that:
\begin{gather*}
    5|X|^4 \exp(-c|Y| \kappa^2 \sigma^4 c_2^2 / (C_1^4 + C_2^4)) \le \frac{\theta}{4}, \\
    2|X|^2 \exp(-|Y| c_1^2/32) \le \frac{\theta}{4}, \\
    2|X||Y| \exp(-\frac{c}{2} \alpha N \phi \cdot \sigma^2/C_1^2) \le \frac{\theta}{4}, \\
    2|X||Y| \exp(-\frac{3}{32} \alpha N \phi) \le \frac{\theta}{4}.
\end{gather*}
One can check that the choices of $N_0$ and $N$ in the theorem statement work.

\subsection{Recovering True Distances using Rulers}
\label{section:findingdistancesfromclusters}
\subsubsection{Non-Missing Case}
Let $Y$ be a set. Suppose we have the following oracle:
\begin{oracle}\label{oracle:ball}
    Fix $\eps > 0$. We have an oracle $\O_1 (x, y) \in \R$ such that if $x, y, z \in Y$ and $d(x, y) \ge d(x, z) + \eps$, then $\O_1 (x, y) > \O_1 (x, z)$.
\end{oracle}
\noindent Using \Cref{oracle:ball}, we can solve our original problem.
\begin{theorem}\label{thm:final recovery}
    Let $Y$ be a diameter $1$ metric space satisfying: if $x, y \in Y$ and $a, b \ge 0$ satisfy $a + b = d(x, y)$, then there exists $z \in Y$ satisfying $|d(x, z) - a| \le \eps$ and $|d(y, z) - b| \le \eps$. Assume \Cref{oracle:ball} for $Y$ and a given $0 < \eps \leq 2^{-20}$. Then for any $x, y \in Y$, we can recover $d(x, y)$ up to $O(\eps \log \eps^{-1})$. Computing all pairwise distances by the proof below takes $O(|Y|^3\log \eps^{-1})$ oracle time.
\end{theorem}
\noindent In fact, we will need an intermediate oracle.
\begin{oracle}\label{oracle:midpoint}
    Fix $\eps > 0$. For any $x, y \in Y$, we can find $z \in Y$ such that $\frac{d(x, y)}{2} - \eps \le d(x, z), d(y, z) \le \frac{d(x, y)}{2} + \eps$ in time $|Y|$.
\end{oracle}

\begin{proposition}\label{prop:oracle 1 implies 2}
    Let $Y$ be a metric space satisfying: if $x, y \in Y$ and $a, b \ge 0$ satisfy $a + b = d(x, y)$, then there exists $z \in Y$ satisfying $|d(x, z) - a| \le \eps$ and $|d(y, z) - b| \le \eps$. Then \Cref{oracle:ball} for $\eps$ implies \Cref{oracle:midpoint} for $\frac{9\eps}{2}$.
\end{proposition}

\begin{proof}
    For any $x, y \in Y$, we choose $z \in Y$ to minimize $\O_1 (x, z)$ subject to $\O_1 (z, y) < \O_1 (z, x)$. This $z$ satisfies $d(z, y) \le d(z, x) + \eps$ (by contrapositive of oracle assumption), so $d(x, z) \ge \frac{d(x, z) + d(y, z) - \eps}{2} \ge \frac{d(x, y) - \eps}{2}$. On the other hand, there exists $z^* \in Y$ such that $|d(x, z^*) - \frac{d(x, y)}{2} - \frac{3\eps}{2}| \le \eps$ and $|d(y, z^*) - \frac{d(x, y)}{2} + \frac{3\eps}{2}| \le \eps$. Then $\O_1 (z^*, y) < \O_1 (z^*, x)$, so $\O_1 (x, z) \le \O_1 (x, z^*)$, which means $d(x, z) \le d(x, z^*) + \eps \le \frac{d(x, y)}{2} + \frac{7\eps}{2}$. Hence, $\frac{d(x, y) - \eps}{2} \le d(x, z) \le \frac{d(x, y) + 7\eps}{2}$. Combined with $d(x, y) - d(x, z) \le d(y, z) \le d(z, x) + \eps$, this implies $\frac{d(x, y) - 7\eps}{2} \le d(y, z) \le \frac{d(x, y) + 9\eps}{2}$.
\end{proof}

\begin{oracle}\label{oracle:ratio}
    For any $x, y, z \in Y$, we can find a real number $r$ such that $|\min \{ d(x, z), d(x, y) \} - r d(x, y)| \le \eps$ in $|Y| \log \eps^{-1}$ time.
\end{oracle}

\begin{proposition}\label{prop:oracle 2 implies 3}
    \Cref{oracle:ball} for $\eps \leq 2^{-20}$ implies \Cref{oracle:ratio} for $O(\eps \log \eps^{-1})$.
\end{proposition}

\begin{proof}
    By Proposition \ref{prop:oracle 1 implies 2}, we may additionally assume Oracle \ref{oracle:midpoint} for $4.5\eps$. Call the respective oracles $\O_1$ and $\O_2$. 
    Setting $n = \lceil\log \eps^{-1}\rceil$, define a function $f : [0, 1] \cap 2^{-n} \Z \to Y$ by $f(0) = x$, $f(1) = y$, and recursively define $f(a + 2^{-k-1}) = \O_2 (f(a), f(a + 2^{-k}))$ whenever $a \in 2^{-k} \Z$. Then using \Cref{oracle:midpoint} we inductively see that
    \begin{equation}\label{eqn:adjacent}
        |d(f(a), f(a + 2^{-k})) - 2^{-k} d(x, y)| \le 9\eps
    \end{equation}
    whenever $a \in 2^{-k} \Z$. Indeed, the base case $k = 0$ is obvious, and if true for $k$, then for $k+1$,
    \begin{align*}
        d(f(a), f(a + 2^{-k-1})) &= d(f(a), \O_2(f(a), f(a + 2^{-k})))  \\
                &\in [\frac{1}{2} d(f(a), f(a+2^{-k})) - 4.5\eps, \frac{1}{2} d(f(a), f(a+2^{-k})) + 4.5\eps] \\
                &\in [2^{-(k+1)} d(x, y) - 9\eps, 2^{-(k+1)} d(x, y) + 9\eps].
    \end{align*}
    Thus, by using \eqref{eqn:adjacent} and the binary representation of $a = \sum_{i=1}^{k_a} 2^{-a_i} \in 2^{-n} \Z$, we have
    \begin{gather*}
        d(x, f(a)) \le \sum_{i=1}^{k_a} d(f(\sum_{j=1}^{i-1} 2^{-a_j}), f(\sum_{j=1}^{i} 2^{-a_j}))\le a d(x, y) + 9k_a \eps \le a d(x, y) + 9n \eps, \\
        d(y, f(a)) \le (1-a) d(x, y) + 9k_{1-a} \eps \le (1-a) d(x, y) + 9n\eps,
    \end{gather*}
    where $k_a$ is the number of ones in the binary representation of $a$. By pitting this against the bound $d(x, y) \le d(x, f(a)) + d(y, f(a))$, we obtain
    \begin{equation}\label{eqn:d - ad}
        |d(x, f(a)) - a d(x, y)| \le 9n \eps.
    \end{equation}
    Now choose the largest $a \in 2^{-n} \Z$ such that $\O_1 (x, z) > \O_1 (x, f(a))$; we claim that $|\min \{ d(x, y), d(x, z) \} - a d(x, y)| \le 9n\eps + 16\eps$. If $a = 1$, then $\O_1 (x, z) > \O_1 (x, y)$, so $d(x, z) \ge d(x, y) - \eps$, which means $|\min \{ d(x, y), d(x, z) \} - a d(x, y)| \le \eps$. If $a < 1$, then $\O_1 (x, f(a + 2^{-n})) \ge \O_1 (x, z) > \O_1 (x, f(a))$, which forces
    $$
    d(x, f(a + 2^{-n})) + \eps \ge d(x, z) \ge d(x, f(a)) - \eps.
    $$
    Using \eqref{eqn:d - ad} shows that $|d(x, z) - a d(x, y)| \le (9n+1)\eps$. This also shows that $|\min \{ d(x, y), d(x, z) \} - a d(x, y)| \le (9n+16)\eps$, as desired. Finally, the $|Y| \log \eps^{-1}$ runtime is from binary search.
\end{proof}

\begin{lemma}
\label{lem:uncapped-ratio-oracle}
Assume \Cref{oracle:ratio} holds with error $\eta$ in a diameter-one metric space. If $d(x,y)\ge D_0>0$, then for every $z\in Y$ one can find $R\in\R$ such that
\[
    |d(x,z)-R d(x,y)|\le 100D_0^{-2}\eta
\]
in $O(|Y|\log \eta^{-1})$ oracle time.
\end{lemma}

\begin{proof}
If $\eta \ge \frac{D_0^2}{100}$, then just pick $R = 0$, as the RHS is just $1$. So assume $\eta < \frac{D_0^2}{100}$.
Use \Cref{oracle:ratio}. If $r \le 1-D_0^{-1} \eta$, then return $r$. If $r > 1-D_0^{-1} \eta$, then apply \Cref{oracle:ratio} to $(x,z,y)$, giving a new number $r'$ and return $\frac{1}{r'}$. This algorithm works because if $r \le 1-D_0^{-1} \eta$, then $\min\{ d(x, z), d(x, y) \} = d(x, z)$. If $r > 1 - D_0^{-1} \eta$, then $|\min\{ d(x,z), d(x,y) \} - d(x,y)| \le D_0^{-1} \eta$. Also $|\min\{ d(x,z), d(x,y) \} - r' d(x,z)| \le \eta$, so $|d(x, y) - r' d(x, z)| \le (D_0^{-1} + 1) \eta$. Since $d(x, z) \le 1$, this forces $r' \ge r' d(x, z) \ge D_0 - (D_0^{-1} + 1) \eta \ge \frac{D_0}{2}$ and
\begin{equation*}
|d(x, z) - \frac{1}{r'} d(x, y)| \le (r')^{-1} (D_0^{-1} + 1) \eta \le 100D_0^{-2} \eta. \qedhere
\end{equation*}
\end{proof}

\begin{proof}[Proof of Theorem \ref{thm:final recovery}]
    Use Propositions \ref{prop:oracle 1 implies 2} and \ref{prop:oracle 2 implies 3} to get oracles $\O_1, \O_2, \O_3$.
    Choose any $x_1 \in Y$, and let $x_2 \in Y$ maximize $\O_1 (x_1, x_2)$. Note that $d(x_1, x_2) \ge \frac{1}{2} - \eps$. Let $x' = \O_2 (x_1, x_2)$. For any $y \in Y$, we test if $\O_1(x_1, y) \le \O_1(x_1, x')$; if this is true, then set $x = x_2$; otherwise, set $x = x_1$.

    Claim. $d(x, y) \ge \frac{1}{4} - 7\eps$.

    \textit{Proof.} If $\O_1(x_1, y) \le \O_1(x_1, x')$, then $d(x_1, y) \le d(x_1, x') + \eps \le \frac{d(x_1, x_2)}{2} + 5.5\eps$, hence $d(x_2, y) \ge \frac{d(x_1, x_2)}{2} - 5.5\eps \ge \frac{1}{4} - 7\eps$.

    If $\O_1(x_1, y) > \O_1(x_1, x')$, then $d(x_1, y) \ge d(x_1, x') - \eps \ge \frac{d(x_1, x_2)}{2} - 5.5\eps \geq \frac{1}{4} - 7\eps$. This proves the Claim.

    Now fix $z\in Y$. By \Cref{lem:uncapped-ratio-oracle}, applied with $D_0=1/8$ and base pair $(y,x)$, we obtain $R_{y,z}$ such that
    \[
        |d(y,z)-R_{y,z}d(y,x)|\le O(\eps\log\eps^{-1}).
    \]
    Applying the same lemma to the base pair $(x_1,x_2)$ gives an estimate of $d(y,x)/d(x_1,x_2)$ with the same additive accuracy, since $d(x_1,x_2)\ge1/2-\eps$. Multiplying the two estimates gives $d(y,z)/d(x_1,x_2)$ up to additive error $O(\eps\log\eps^{-1})$. Hence we can identify a pair maximizing the diameter of $Y$ up to this additive error and normalize by the diameter-one assumption to recover all $d(y,z)$ up to $O(\eps\log\eps^{-1})$.
\end{proof}


\subsubsection{Missing Data Case}

For the missing data case, we need a slightly stronger oracle.
\begin{theorem}\label{thm:recover manifold from missing noisy data}
    Suppose $r \in (0, \frac{1}{100}]$ and let $(Y, d)$ be a diameter $D \in [r, 1]$ metric space satisfying the following ``geodesic'' property: for every $x, y \in Y$ and numbers $a_1, \cdots, a_n \ge 0$ with $\sum_{i=1}^n a_i = d(x, y)$, there exist $x = x_0, x_1, x_2, \cdots, x_n = y \in Y$ such that $|d(x_{i-1}, x_i) - a_i| \le \eps$ for each $1 \le i \le n$. Suppose also that there exist numbers $A(x, y) \in \R \cup \{ +\infty \}$ such that:
    \begin{enumerate}
        \item for every $x, y \in Y$, if $d(x, y) \le r$, then $A(x, y) \neq \infty$;
        
        \item for every $x, y, z \in Y$ satisfying $d(x, y) \ge d(x, z) + \eps$ and $A(x, y), A(x, z) \neq +\infty$, we have $A(x, y) > A(x, z)$.
    \end{enumerate}
    If $\eps < \frac{r^2 D}{100}$, then we can determine the distances $d(x, y)$ up to additive error $O(\frac{\eps}{r^2} \log \eps^{-1})$. A direct implementation of the proof takes $O(|Y|^5+|Y|^4\eps^{-1})$ time after the values $A(x,y)$ have been computed.
\end{theorem}

\begin{remark}
    In fact, we can do slightly better: error $O(\frac{\eps D^2}{r^2} \log \eps^{-1})$ and only need $\eps < \frac{r^2}{100}$. This is by taking paths of length $\le \frac{3D}{r}$ instead of $\le \frac{3}{r}$ in the ensuing argument. However, when $D = 1$ this is the same bound, and furthermore the theorem can handle when we don't know the diameter $D$ (but know it lies in $[r, 1]$).
\end{remark}

\begin{remark}
    In fact, the ``geodesic'' property for $n$ with $2\eps$ follows from the geodesic property for $\eps$ by induction, similarly to \eqref{eqn:adjacent}.
\end{remark}

\begin{proof}
Define a directed graph $G$ on pairs $(a, b) \in Y \times Y$ where we connect $(a, b)$ to $(b, c)$ by a directed edge if $\infty > A(b, a) > A(b, c)$. Note that $G$ has $|Y|^2$ vertices and $|Y|^3$ edges, so Dijkstra's shortest path algorithm (on the unweighted directed graph $G$) operates in time $O(|Y|^3)$.

Now, for all $x \in Y$, define $\tau(x) \in Y$ as the minimizer of $A(x, \tau(x))$ subject to for every $z \in Y$, there exists $z' \in Y$ and a directed path from $(\tau(x), x)$ to $(z', z)$ with length at most $\frac{3}{r}$.

\begin{claim}
\label{claim:taux}
$d(x, \tau(x)) \in [\frac{rD}{9}, \frac{rD}{2}]$
\end{claim}

\begin{proof}

Choose any $x' \in Y$ such that $d(x, x') \in [\frac{2rD}{5} - \eps, \frac{2rD}{5} + \eps]$. We show that for every $z \in Y$, there exists $z' \in Y$ and a directed path from $(x', x)$ to $(z', z)$ with length at most $\frac{3}{r}$. Indeed, choose $n = \lfloor \frac{3}{r} \rfloor$ and a sequence $x = x_0, x_1, x_2, \cdots, x_n = z$ consisting of points in $Y$ such that
$$|d(x_{i-1}, x_i) - \frac{d(x, z)}{n} - (\frac{n+1}{2} - i) \cdot 3\eps| \le \eps.$$
Plugging in $i = 1$, we deduce that (since $\frac{3}{r} \ge n \ge \frac{3}{r} - 1 \ge \frac{2.99}{r}$ and $\eps < \frac{r^2 D}{100}$)
$$d(x', x) \ge \frac{2rD}{5} - \eps > \frac{D}{n} + \frac{3n+1}{2} \eps \ge d(x, x_1) + \eps,$$
hence $A(x, x') > A(x, x_1)$ and there is an edge from $(x', x)$ to $(x_0, x_1)$. Similarly, there exists an edge from $(x_{i-1}, x_i)$ to $(x_i, x_{i+1})$ for all $1 \le i \le n-1$, showing the existence of the directed path from $(x', x)$ to $(z', z)$.

By minimality, we have $A(x, \tau(x)) < A(x, x')$, so $d(x, \tau(x)) \le d(x, x') + \eps \le \frac{rD}{2}$.

Now, we show $d(x, \tau(x)) \ge \frac{rD}{9}$. Choose $z$ such that $d(x, z) \ge \frac{D}{2}$ (this exists since our metric space has diameter $D$); we know there exists a directed path from $(\tau(x), x)$ to $(z', z)$ with length at most $\frac{3}{r}$, which corresponds to a $x = x_0, x_1, x_2, \cdots, x_n = z$ consisting of points in $Y$ such that $A(x_i, x_{i-1}) > A(x_i, x_{i+1})$, hence $d(x_i, x_{i+1}) \le d(x_{i-1}, x_i) + \eps$. By the triangle inequality, we have
$$
n d(x, \tau(x)) + \frac{n(n+1)}{2} \eps \ge \frac{D}{2}.
$$
This forces $d(x, \tau(x)) \ge \frac{D}{3n} \ge \frac{rD}{9}$, proving the claim.
\end{proof}

Using this, for any $x, y \in Y$ with $d(x, y) \le r$, we can find an approximate midpoint $z \in Y$ satisfying $\frac{1}{2} d(x, y) - 4.5\eps \le d(x, z), d(y, z) \le \frac{1}{2} d(x, y) + 4.5\eps$ in time $|Y|$ (\Cref{prop:oracle 1 implies 2}).

Repeatedly using this approximate midpoint oracle and assuming $\eps$ is a negative power of $2$ for simplicity, we can find a function $\Gamma_x : [0, 1] \cap \eps \Z \to Y$ such that
\begin{equation}\label{eqn:gamma_x}
    |d(x, \Gamma_x (\eps k)) - \eps k d(x, \tau(x))| = O(\eps \log \eps^{-1})
\end{equation}
for each $0 \le k \le \eps^{-1}$ in $\eps^{-1}$ time (see proof of \Cref{prop:oracle 2 implies 3}). Thus, for any $x, y \in Y$, we can use binary search to compare $A(x, y)$ to each $A(x, \Gamma_x (k\eps))$ for $0 \le k \le \eps^{-1}$, so we can find $\min \{ \frac{d(x, y)}{d(x, \tau(x))}, 1 \}$ up to error $O(\frac{\eps \log \eps^{-1}}{r})$ in time $\log \eps^{-1}$.

Using this approximate distance ratio oracle, for any $x, y \in Y$, we can find $\frac{d(x, y)}{d(x, \tau(x))}$ up to error $O(\frac{\eps \log \eps^{-1}}{rD} \left( 1 + \frac{d(x, y)}{d(x, \tau(x))} \right)^2)$. Indeed, if $A(x, y) < A(x, \tau(x))$, then we use the previous oracle; otherwise, we make use of the following claim.

\begin{claim}
\label{claim:1323}
$\forall r \geq \frac{1}{3} \exists m \in \Na, \ell \in [\frac{1}{3}, \frac{2}{3}]$ such that $r = m \ell$.
\end{claim}

\begin{proof}
To prove this claim, first observe that if $r \leq \frac{2}{3}$ then the claim is trivial. Therefore, assume that $r = \frac{r}{1} > \frac{2}{3}$. Since the sequence $(\frac{r}{n})_{n \in \Na}$ is strictly decreasing, we eventually have some $n$ such that $\frac{r}{n} > \frac{2}{3}$ and $\frac{r}{n+1} \leq \frac{2}{3}$. Observe that the former statement yields $\frac{r}{n+1} > \frac{n}{n+1} \frac{2}{3} \ge \frac{1}{3}$ for all $n \ge 1$. Hence, $\ell = \frac{r}{n+1}$ satisfies the lemma with $m = n+1$, proving the claim.
\end{proof}

Based on \Cref{claim:1323}, we will show that $\frac{d(x,y)}{d(x,\tau(x)}$ can be approximated by the minimum value of $k\eps \cdot n_k$ over all $k\eps \in [\frac{1}{3}, \frac{2}{3}]$, where $n_k$ is the length of the shortest path from $(\Gamma_x (k\eps), x)$ to some $(y', y)$. The proof follows a similar thought process as in \Cref{claim:taux}. Choose $n$ such that $\frac{d(x, y)}{n} \in [\frac{d(x, \tau(x)}{3}, \frac{2d(x, \tau(x)}{3}]$, and choose a sequence $x = x_0, x_1, x_2, \cdots, x_n = y$ consisting of points in $Y$ such that
$$|d(x_{i-1}, x_i) - \frac{d(x, y)}{n} - (\frac{n+1}{2} - i) \cdot 3\eps| \le \eps.$$
Note that for

\begin{equation}
\label{eq:kepscompatn}
    k\eps \cdot d(x, \tau(x)) > \frac{d(x, y)}{n} + \frac{n+1}{2}\eps + \frac{C\eps \log \eps^{-1}}{rD}
\end{equation}
we have $d(\Gamma_x (k\eps), x) > \frac{d(x, y)}{n} + \frac{n+1}{2} \eps \ge d(x, x_1) + \eps$, hence $A(x, \Gamma_x(k\eps)) > A(x, x_1)$ and there is an edge from $(\Gamma_x(k\eps), x)$ to $(x_0, x_1)$. Similarly, there exists an edge from $(x_{i-1}, x_i)$ to $(x_i, x_{i+1})$ for all $1 \le i \le n-1$, showing the existence of the directed path from $(\Gamma_x(k\eps), x)$ to $(y', y)$. Now, assume that $k$ is the minimal possible value such that \Cref{eq:kepscompatn} holds. This implies that $k\eps \cdot d(x,\tau(x)) \le \frac{d(x,y)}{n} + \frac{n+1}{2}\eps + \frac{C\eps \log \eps^{-1}}{rD} + \eps d(x,\tau(x))$. Dividing by $\frac{d(x,\tau(x))}{n}$ on both sides and using $\frac{}{rD} \leq n \leq \frac{3d(x,y)}{d(x,\tau(x))}$ yields

\[
nk\eps \le \frac{d(x, y)}{d(x, \tau(x))} + O(\frac{n^2 \eps \log \eps^{-1}}{rD}) = \frac{d(x, y)}{d(x, \tau(x))} + O(\frac{\eps \log \eps^{-1}}{rD} \left( 1 + \frac{d(x, y)}{d(x, \tau(x))} \right)^2).
\]
with the last inequality following from $n \leq 3(1+\frac{d(x,y)}{d(x,\tau(x))})$. Now, we show $nk\eps \ge \frac{d(x, y)}{d(x, \tau(x))} - O(\frac{\eps \log \eps^{-1}}{rD} \left( 1 + \frac{d(x, y)}{d(x, \tau(x))} \right)^2)$. We know there exists a directed path from $(\Gamma_x (k\eps), x)$ to $(y', y)$ with length at most $n$, which corresponds to a $x = x_0, x_1, x_2, \cdots, x_n = y$ consisting of points in $Y$ such that $A(x_i, x_{i-1}) > A(x_i, x_{i+1})$, hence $d(x_i, x_{i+1}) \le d(x_{i-1}, x_i) + \eps$. By \eqref{eqn:gamma_x} and the triangle inequality, we have
\begin{equation}\label{eqn:nkeps}
nk\eps d(x, \tau(x)) + (O(n \log \eps^{-1}) + \frac{n(n+1)}{2}) \eps \ge n d(x, \Gamma_x (k\eps)) + \frac{n(n+1)}{2} \eps \ge d(x, y).
\end{equation}
If $nk\eps \ge \frac{d(x, y)}{d(x, \tau(x))}$, then we have an even stronger statement than desired; otherwise, if $nk\eps < \frac{d(x, y)}{d(x, \tau(x))}$, then we get from $k\eps \in [\frac{1}{3}, \frac{2}{3}]$ that $n \le \frac{3d(x, y)}{d(x, \tau(x))}$, so rearranging \eqref{eqn:nkeps} gives $nk\eps \ge \frac{d(x,y)}{d(x, \tau(x))} - O(\frac{\eps \log \eps^{-1}}{rD} \left( 1 + \frac{d(x, y)}{d(x, \tau(x))} \right)^2)$, as desired.

Now, to compute $\frac{d(x, \tau(x))}{d(z, \tau(z))}$ for any $x, y, z$, we note that if $A(z, x) = \infty$ or $A(z, x) > A(z, \tau(x))$, then $d(z, x) \ge r$ in the first case and $d(z, x) \ge d(z, \tau(x)) - \eps \ge d(x, \tau(x)) - d(z, x) - \eps$ implies $d(z, x) \ge \frac{rD}{20}$ in the second case, and therefore we can compute $\frac{d(x, \tau(x))}{d(z, \tau(z))} = \frac{d(z, x)}{d(z, \tau(z))} / \frac{d(z, x)}{d(x, \tau(x))}$. Otherwise if $A(z, \tau(x)) = \infty$ or $A(z, \tau(x)) \ge A(z, x)$, then $d(z, \tau(x)) \ge \frac{rD}{20}$, and therefore we can compute this as $\frac{d(z, \tau(x))}{d(z, \tau(z))} / \frac{d(z, \tau(x))}{d(x, \tau(x))}$.
The additive error is $O(\frac{\eps}{r^2 D} \log \eps^{-1})$ since it is the quotient of two terms of the form $n \pm \frac{n^2 \eps \log \eps^{-1}}{rD}$ where $n \sim \frac{d(z, x)}{d(z, \tau(z))}$, $\frac{1}{20} \le n \le \frac{10}{r}$. Also, note that $\frac{d(x, \tau(x))}{d(z, \tau(z))} \in [\frac{1}{100}, 100]$.

Finally, to compute $d(x, y)$, we use $\frac{d(x, y)}{d(x, \tau(x))} \cdot \frac{d(x, \tau(x))}{d(z, \tau(z))} \cdot d(z, \tau(z))$. For a fixed $z$, this allows us to compute all $d(x, y)$ up to error $O(\frac{\eps \log \eps^{-1}}{r^2})$ (since $\frac{d(x, y)}{d(x, \tau(x))} \le \frac{10}{r}$). Use the diameter $1$ assumption to normalize.
\end{proof}

\subsection{Proofs of Theorems \ref{thm:mainresthm1} and \ref{thm:mainresthm2}}
The basic idea is to combine \Cref{prop:density}, \Cref{thm:main_technical}, and either \Cref{thm:final recovery} or \Cref{thm:recover manifold from missing noisy data}.

We give some more details for the missing case Theorem \ref{thm:mainresthm2}; the non-missing case Theorem \ref{thm:mainresthm1} is easier. Recall from \Cref{thm:mainresthm2} the constants $d = d_\mu \ge 2$, $c, \kappa > 0, C, C_1, C_2 \ge 1$, $0 < r_0 \le 1$, $\eps, \theta > 0$. We will use these constants to define new parameters used in \Cref{thm:main_technical}. Define $\VR(r) = cr^d$ to be the lower estimate of the volume of a $\mu$-ball from assumption (A1) of Theorem \ref{thm:mainresthm1}. 

Set $c_1 = \phi^2 \VR(\frac{r_0}{2}) = c \phi^2 (r_0/2)^d$. If $\eps > \frac{1}{3} \lambda_2 c_1$, replace $\eps$ by $\frac{1}{3} \lambda_2 c_1$.

Set $\kappa^* = \frac{1}{2} \min \{ \VR(\frac{\lambda_2 c_1}{8}), \kappa \} = \frac{1}{2} \min \{ c \left( \frac{\lambda_2 c_1}{8} \right)^d, \kappa \}^d$, $c_2 = \lambda_1^4 c_1^4$, $\sigma = \frac{\eps}{C}$, $\delta = \frac{\kappa^* \sigma^2}{8C_2 C}$, $\alpha = \frac{1}{4} \VR(\delta) = \frac{1}{4} c \delta^d = c \frac{(\kappa^*)^d \eps^{2d}}{4 \cdot 8^d C_2^d C^{3d}}$, and finally
\begin{gather*}
N = \frac{8}{\alpha} \log \frac{8}{\theta \alpha} + \frac{CC_1^2}{\alpha \phi \sigma^2} \log \frac{N_0}{\theta} \sim \frac{CC_1^2}{\alpha \phi \sigma^2} \log \frac{N_0}{\theta} \sim (\kappa^*)^{-d} \eps^{-2d-2} (d \log \eps^{-1} + \log \theta^{-1}), \\
N_0 = \frac{1}{\VR(\eps/2)} \log \frac{1}{\VR(\eps/4) \theta} + \frac{1}{\VR(r_0/2)} \log \frac{N}{\theta} + \frac{8}{\kappa^*} \log \frac{N^2}{\theta} \lesim (\kappa^*)^{-1} \eps^{-\max \{ d, 4 \}} \left(d \log \frac{1}{\eps} + \log \frac{1}{\theta} \right).
\end{gather*}

By applying \Cref{prop:density} and \Cref{prop:chernoff}, we can ensure for a random i.i.d sample from $X$ with $|X| = N$ and $Y$ be the first $N_0$ elements of $X$:
\begin{itemize}
    \item every ball $B(x, \delta)$ with $x \in Y$ contains at least $\frac{1}{2}VR(\delta) (N-N_0) \ge \alpha N$ many elements of $X \setminus Y$ with probability $\ge 1 - |Y| e^{-\frac{1}{8} \VR(\delta) |X\setminus Y|}$;

    \item $Y$ is $\eps$-dense in $\M$ with probability $\ge 1 - \frac{1}{\VR(\eps/4)} e^{-\frac{1}{8} \VR(\frac{\eps}{2}) |Y|}$;

    \item every ball $B(x, \frac{r_0}{2})$ with $x \in X$ contains at least $\frac{1}{2} \VR(\frac{r_0}{2}) |Y|$ many elements of $Y$ with probability $\ge 1 - |X| e^{-\frac{1}{8} \VR(\frac{r_0}{2}) |Y|}$;

    \item every $\Lambda^1_{x,y} := \{ z \in B(x, \frac{r_0}{2}) : |d(x, z) - d(y, z)| \ge \frac{d(x, y)}{4} \}$ with $x, y \in X$, $d(x, y) \le \frac{r_0}{3}$ contain at least $\kappa |Y|$ many elements of $Y$ with probability $\ge 1 - |X|^2 e^{-\frac{1}{8} \kappa |Y|}$;

    \item for every $x, y \in X$, if $d(x, y) \ge \frac{r_0}{3}$, then $\Lambda^2_{x,y} := \{ z : d(x, z) - \frac{d(x,y)}{2} \in [\frac{\lambda_2 c_1}{4}, \frac{3\lambda_2 c_1}{4}], d(y, z) - \frac{d(x,y)}{2} \in [-\frac{\lambda_2 c_1}{4}, -\frac{3\lambda_2 c_1}{4}] \}$ contains at least $\kappa^* |Y|$ many elements of $Y$ with probability $\ge 1 - |X|^2 e^{-\frac{1}{8} \kappa^* |Y|}$.
\end{itemize}
To justify the last bullet point, we need to see that $\Lambda^2_{x,y}$ contains a ball of radius $\frac{\lambda_2 c_1}{4}$ centered at a point which is exactly distance $\frac{d(x, y)}{2} + \frac{\lambda_2 c_1}{2}$ from $x$ and $\frac{d(x, y)}{2} - \frac{\lambda_2 c_1}{2}$ from $y$.

The union bound shows that all bullet points hold simultaneously with probability
\[
    1- |Y|e^{-\frac{1}{8}VR(\delta)|X \setminus Y|} - \frac{1}{\VR(\frac{\eps}{4})}e^{-\frac{1}{8} \VR(\frac{\eps}{2})|Y|} - |X|e^{-\frac{1}{8}\VR(\frac{r_0}{2})|Y|} - 2|X|^2 e^{-\frac{1}{8}\kappa^* |Y|} \ge 1-\frac{\theta}{2}.
\]
Finally, we apply \Cref{thm:main_technical} with the above parameters. The key points to note:
\begin{itemize}
    \item If $x, y \in X$ are distinct and satisfy $d(x, y) < \delta$, then if $z \in B(x, r_0 - \eps)$, then $p(x, z), p(y, z) > \phi$.

    \item If $d(x, y) \le \frac{r_0}{3}$, then since $c_2 \le \phi^4$, we have $\Lambda_{x,y,v,w} \supset \Lambda^1_{x,y}$. If $d(x, y) > \frac{r_0}{3}$ and $\frac{1}{|Y|} \sum_{z \in Y} p(x, z) p(y, z) > c_1$, then $p(x, z') p(y, z') > c_1$ for some $z' \in Y$. Without loss of generality let $d(x, z') \ge \frac{d(x, y)}{2}$; then $p(x, z') > c_1$ (since $p(y, z') \le 1$). The condition on $p$ from \Cref{thm:main} then implies for all $v \in B(x, \delta), w \in B(y, \delta)$, and $z \in \Lambda^2_{x,y}$, we have
    \[
        d(x, z), d(y, z), d(v, z), d(w, z) \le \frac{d(x, y)}{2} + \frac{3\lambda_2 c_1}{4} + \delta \le d(x, z') + \lambda_2 c_1 \le d(x, z') + \lambda_2 p(x, z').
    \]
    Thus, by condition 2 of missing data assumption in \Cref{thm:main}, we have
    \[
        p(x, z), p(y, z), p(v, z), p(w, z) \ge \lambda_1^4 p(x, z')^4 > \lambda_1^4 c_1^4.
    \]
    $p(w, z)$ and $p(v, z) > \lambda_1 p(x, y)$. Finally, $\eps \le \frac{1}{3} \lambda_2 c_1$ implies $\Lambda_{x,y,v,w} \supset \Lambda^2_{x,y}$.

    \item Our choice of $N_0$ and $N$ satisfy the assumption because of the $\max \{ d, 4 \}$ in the definition of $N_0$.
\end{itemize}
Finally, the hypothesis regarding $Y$ of \Cref{thm:final recovery} and \ref{thm:recover manifold from missing noisy data} is satisfied because $Y$ is $\eps$-dense in $\M$, which is a geodesic space. Also $Y$ has diameter between $1-2\eps$ and $1$.

For every $x \in X$, we define $\pi(x)$ to be the $y \in Y$ that minimizes $A(x, y)$; then $d(x, \pi(x)) \le 18\eps$ and so we can estimate $d(x_1, x_2)$ via $d(\pi(x_1), \pi(x_2))$.

\section{Observer and Observed Regions}
\label{sec:partialobs}
Split the observer sample into two independent parts $X_\nu=W_\nu\sqcup Z_\nu$.
We will use $W_\nu$ for the inner-product cluster construction and $Z_\nu$ for the ruler construction. This split is only a device for the proof: it ensures that the variables later averaged over clusters are independent of the variables used to define the clusters.

\subsection{Inner Product Framework}
\label{sec:partial-inner-products}
Define
\[
    \rho:=\frac{\eps}{9},
    \sigma:=\frac{\rho}{C},
    \kappa_*:=\frac{\kappa}{2}, \delta:=\frac{\kappa_*\sigma^2}{8C_2C},
    \alpha:=\frac{c}{8}\delta^d.
\]
Here $c$ is the lower-Ahlfors constant for $\mu$ and $\kappa$ is from (D1).

\begin{definition}[Partial-observation clusters]
\label{def:partialobs-cluster}
For $x\in X_\mu$, define
\[
\C(x):=\left\{y\in X_\mu:
\sup_{z\in X_\mu\setminus\{x,y\}}
\left|
\frac1{|W_\nu|}\sum_{q\in W_\nu}
(d'(x,q)-d'(y,q))d'(z,q)
\right|
\le \frac{3}{16}\kappa_*\sigma^2
\right\}.
\]
\end{definition}

\begin{proposition}[Partial-observation cluster construction]
\label{prop:partialobs-clusters}
Assume the following sample events hold:
\begin{enumerate}[(i)]
    \item every ball $B_\M(x,\delta)$ with $x\in X_\mu$ contains at least $\alpha N_\mu$ points of $X_\mu$;
    \item for all $x,y\in X_\mu$ with $d(x,y)>4\rho$, we have $\frac1{|W_\nu|}\left|\{q\in W_\nu:
        |d(x,q)- d(y,q)|\ge\rho\}\right|\ge \kappa_*$.
\end{enumerate}
If $|W_\nu|\ge
    K_1\frac{C_1^4+C_1^2C_2^2+C_2^4}{\kappa_*^2\sigma^4}
    \log\frac{N_\mu}{\theta}$,
then, with probability at least $1-\theta$,
\[
    B_\M(x,\delta)\cap X_\mu\subseteq\C(x)\subseteq B_\M(x,4\rho)\cap X_\mu
    \qquad(x\in X_\mu).
\]
After precomputing the $I(x,z)$ in time $O(N_\mu^2|W_\nu|)$, all clusters can be computed in time $O(N_\mu^3)$.
\end{proposition}

\begin{proof}
This is the same argument as \Cref{prop:cluster}, with the cluster definition in \Cref{def:cluster} replaced by \Cref{def:partialobs-cluster} and with $X=X_\mu$, $Y=W_\nu$. The concentration input is \Cref{prop:hoeffding ip}. The functional separation step is exactly \Cref{lem:sep}. The only change is that the lower bound on $\|f_x-f_y\|_2^2$ for $d(x,y)>4\rho$ is now supplied by condition (ii): on at least a $\kappa_*$ fraction of $W_\nu$, the difference $|d(x,q)-d(y,q)|$ is at least $\rho$, hence $|F(d(x,q))-F(d(y,q))|\ge\sigma$. Thus $\|f_x-f_y\|_2^2\ge\kappa_*\sigma^2$. If $d(x,y)\le\delta$, the bilipschitz upper bound on $F$ gives
\[
    |f_x(q)-f_y(q)|\le C\delta=\frac{\kappa_*\sigma^2}{8C_2}
    \qquad(q\in W_\nu),
\]
which is the uniform-continuity input used in \Cref{prop:cluster}. The runtime is the same precomputation-plus-supremum count as in \Cref{prop:cluster}.
\end{proof}

For $q\in Z_\nu$ and $x\in X_\mu$, define the cluster average
\[
    \tA(q,x):=\frac1{|\C(x)|}\sum_{v\in\C(x)}d'(v,q).
\]
\begin{proposition}
\label{prop:partialobs-comparison-values}
Assume the conclusions of \Cref{prop:partialobs-clusters} hold and that
\[
    \alpha N_\mu\ge K_2\frac{C_1^2}{\sigma^2}
    \log\frac{N_\mu|Z_\nu|}{\theta}.
\]
Then, with probability at least $1-\theta$ over the noisy observations indexed by $Z_\nu$, the following implication holds for every $q,q'\in Z_\nu$ and $x,x'\in X_\mu$:
\[
    d(q,x)\ge d(q',x')+s
    \quad\Longrightarrow\quad
    \tA(q,x)>\tA(q',x').
\]
The values $\tA(q,x)$ can be computed in time $O(|Z_\nu|N_\mu^2)$ by direct summation.
\end{proposition}

\begin{proof}
Since $W_\nu$ and $Z_\nu$ are disjoint, the cluster memberships are independent of the variables $d'(v,q)$ with $q\in Z_\nu$. Conditional on the clusters, \Cref{thm:hoeffding} and $|\C(x)|\ge\alpha N_\mu$ give, uniformly over $q\in Z_\nu$ and $x\in X_\mu$,
\[
\left|
\tA(q,x)-\frac1{|\C(x)|}\sum_{v\in\C(x)}F(d(v,q))
\right|\le\frac{\sigma}{3}.
\]
If $d(q,x)\ge d(q',x')+s$, then for every $v\in\C(x)$ and $v'\in\C(x')$,
\[
    d(q,v)\ge d(q,x)-4\rho
    \ge d(q',x')+9\rho-4\rho
    \ge d(q',v')+\rho.
\]
Thus $F(d(q,v))-F(d(q',v'))\ge\rho/C=\sigma$. The two concentration losses leave a positive gap, so $\tA(q,x)>\tA(q',x')$.
\end{proof}

\subsection{Rulers}
\label{sec:partial-rulers}

In this subsection we condition on the high-probability events supplied by \Cref{prop:density} and \Cref{prop:partialobs-comparison-values}: $X_\mu\subset\M$ is an $\eps$-dense target sample, $X_\nu\subset Y$ is an $\eps$-dense observer sample, and the comparison values $\tA$ below satisfy the oracle property at scale $\eps$.

\begin{oracle}
\label{ora:partialobs0}
Fix $\eps>0$. We have an oracle $\O_0(x,y)\in\R$ such that if $x,x'\in X_\nu$, $y,z\in X_\mu$, and $d(x,y)\ge d(x',z)+\eps$,
then $\O_0(x,y)>\O_0(x',z)$.
\end{oracle}

\begin{oracle}
\label{ora:partialobs1}
Fix $\eta>0$. We have an oracle $\O_1(x,y)\in\R$ such that if $x,x'\in X_\nu$, $y,z\in X_\nu\cup X_\mu$, and $d(x,y)\ge d(x',z)+\eta$,
then $\O_1(x,y)>\O_1(x',z)$.
\end{oracle}

\begin{lemma}
\label{lem:partialobs-oracle1}
\Cref{ora:partialobs0} at scale $\eps$ implies \Cref{ora:partialobs1} at scale $5\eps$.
\end{lemma}

\begin{proof}
For $x\in X_\nu$, choose $\pi(x)\in\arg\min_{y\in X_\mu}\O_0(x,y)$. Then choose $y_x\in X_\mu$ with $d(x,y_x)\le\eps$. If $d(x,\pi(x))>d(x,y_x)+\eps$, then \Cref{ora:partialobs0} gives $\O_0(x,\pi(x))>\O_0(x,y_x)$, contradicting the definition of $\pi(x)$. Thus, $d(x,\pi(x))\le d(x,y_x)+\eps\le2\eps$.

Define
\[
    \O_1(x,y):=
    \begin{cases}
        \O_0(x,y), & y\in X_\mu,\\
        \O_0(x,\pi(y)), & y\in X_\nu.
    \end{cases}
\]
If $y,z\in X_\nu\cup X_\mu$, write $\bar y=y$ for $y\in X_\mu$ and $\bar y=\pi(y)$ for $y\in X_\nu$, and define $\bar z$ similarly. Since $d(y,\bar y)\le2\eps$ and $d(z,\bar z)\le2\eps$, triangle inequality gives $d(x,\bar y)\ge d(x,y)-2\eps$ and $d(x',\bar z)\le d(x',z)+2\eps$.
Thus $d(x,y)\ge d(x',z)+5\eps$ implies $d(x,\bar y)\ge d(x',\bar z)+\eps$, and \Cref{ora:partialobs0} gives the claim.
\end{proof}

\begin{oracle}
\label{ora:partialobs2}
Fix $\eta>0$. For any $x,y\in X_\nu$, we can find $\O_2(x,y)=z\in X_\mu$ such that
\[
    \frac{d(x,y)}2-\eta\le d(x,z),d(y,z)\le\frac{d(x,y)}2+\eta.
\]
\end{oracle}

The following proposition is exactly analogous to the proof of Proposition \ref{prop:oracle 1 implies 2}.
\begin{proposition}
\label{prop:partialobs1}
\Cref{ora:partialobs0} at scale $\eps$ implies \Cref{ora:partialobs2} at scale $9\eps/2$. The runtime of an oracle call is $|X_\mu|$.
\end{proposition}


\begin{oracle}
\label{ora:partialobs5}
Fix $\eta>0$. Given $x,y\in X_\nu$ and $z\in X_\mu$ satisfying $d(x,z) \le d(x,y)$, we can find $\O_3(x,y;z)=w\in X_\mu$ such that
\[
    \left|d(x,w)-\frac12d(x,z)\right|\le\eta,
    \qquad
    d(x,w)+d(y,w)-d(x,y)\le\eta.
\]
\end{oracle}

\begin{proposition}
\label{prop:partialobs2}
Assume (D3), and assume $X_\nu$ and $X_\mu$ are $\eta$-dense in their respective ambient sets. If \Cref{ora:partialobs1} and \Cref{ora:partialobs2} are available at scale $\eta$, then \Cref{ora:partialobs5} is available at scale $8\eta$ with runtime $O(N_\nu^2+|X_\mu|)$.
\end{proposition}

\begin{proof}
If there is a pair $a,b\in X_\nu$ satisfying $\O_1(a,b)>\O_1(x,z)$, choose the one satisfying $\O_1(a,b)>\O_1(x,z)$ which minimizes $\O_1(a,b)$; otherwise, choose the pair $a,b\in X_\nu$ which maximizes $\O_1(a,b)$. In the minimization case, we have that $d(x, y)-\eta \le d(a, b) \le d(x, y) + 6\eta$.
Indeed, the left inequality is true by \Cref{ora:partialobs1}, and the right inequality is true because otherwise by (D3) and $\eta$-density of $X_\nu$ we can find $a', b' \in X_\nu$ with $d(x, y) + \eta \le d(a', b') \le d(a, b) - \eta$, and then $\O_1 (a, b) > \O_1 (a', b') > \O_1 (x, z)$. In the maximization case, we similarly see $d(x, y)-6\eta \le d(a, b) \le d(x, y) + \eta$. Hence, in either case we have
\begin{equation}\label{eqn:temp1}
    |d(x, y) - d(a, b)| \le 6\eta.
\end{equation}
Let $S=d(a,b)$ and let $c=\O_2(a,b)$. Since \Cref{ora:partialobs2} is available at scale $\eta$, we have $\left|d(a,c)-\frac{d(a,b)}{2}\right|\le\eta$, and combining with \eqref{eqn:temp1} gives
\begin{equation}
\label{eq:partialobs-T-L}
    \left|d(a,c)-\frac{d(x,z)}{2}\right|\le4\eta.
\end{equation}
Finally, choose $w\in X_\mu$ minimizing $\O_1(y,w)$ subject to $\O_1(x,w)<\O_1(a,c)$ (and if such a point doesn't exist, choose $w = \pi(x)$, where $\pi$ was defined in proof of Lemma \ref{lem:partialobs-oracle1}). In the main case where $\O_1(x,w)<\O_1(a,c)$, we clearly have $d(x, w) \le d(a, c) + \eta \le \frac{d(x,z)}{2} + 5\eta$. By minimality of $\O_1 (y, w)$, we also have $d(y, w) \le d(x, y) - d(a, c) + 4\eta$. Indeed, if $d(y, w) \ge d(x, y) - d(a, c) + 4\eta$, then let $w^*$ be a point such that $d(x, w^*) = d(a, c) - 2\eta$ and $d(y, w^*) = d(x, y) - d(a, c) + 2\eta \le d(y, w) - 2\eta$, and let $w'$ be $\eta$-close to $w^*$; then $d(x, w') \le d(a, c) - \eta$ and so $\O_1 (x, w') < \O_1 (a, c)$, and $d(y, w') \le d(y, w) - \eta$ and so $\O_1 (y, w') < \O_1 (y, w)$, contradicting the minimality of $w'$. Thus, we have showed $d(y, w) \le d(x, y) - d(a, c) + 4\eta$.

This means $d(x, w) + d(y, w) \le d(x, y) + 5\eta$. Also, triangle inequality gives $$d(x, w) \ge d(x, y) - d(y, w) \ge d(a, c) - 4\eta \ge \frac{d(x,z)}{2} - 8\eta,$$ so the proposition is proved in the main case. If $w$ doesn't exist, this means $d(a, c) \le 2\eta$, and clearly $w = \pi(x)$ satisfies the conditions.

The runtime is the search for $(a, b)$, one midpoint call, and one optimization over $X_\mu$.
\end{proof}

\begin{proposition}
\label{prop:partialobs3}
Suppose we start from $x_1=1/2$ and, given $x_k$, the admissible values for $x_{k+1}$ are either $x_k/2$ or $(1+x_k)/2$. Then the set of admissible values for $x_k$ is
\[
    T_k=\left\{\frac{q}{2^k}:q\in\{1,3,\dots,2^k-1\}\right\}.
\]
\end{proposition}

\begin{proof}
This follows by induction. The base case is $T_1=\{1/2\}$. If $T_{k-1}$ is the set of odd dyadic fractions with denominator $2^{k-1}$, then applying $s\mapsto s/2$ gives the odd fractions in $(0,1/2)$ with denominator $2^k$, while applying $s\mapsto(1+s)/2$ gives the odd fractions in $(1/2,1)$ with denominator $2^k$.
\end{proof}

\begin{oracle}
\label{ora:partialobs6}
Fix $0<\eta<2^{-20}$. Given $x,y\in X_\nu$, for every $x'\in X_\nu$ and $z\in X_\mu$ we can find $\O_4 (x, y) = r\in\R$ such that
\[
    \left|\min\{d(x',z),d(x,y)\}-r d(x,y)\right|\le\eta.
\]
Each query costs $O(\eta^{-1})$ comparisons by direct search along the ruler.
\end{oracle}

\begin{proposition}
\label{prop:partialobs4}
Assume \Cref{ora:partialobs1} and \Cref{ora:partialobs2} are available at scale $\eta$, and assume \Cref{ora:partialobs5} is available at scale $\eta$. Then \Cref{ora:partialobs6} is available at scale $6\eta$ after one-time work $O((N_\nu^2+|X_\mu|)\eta^{-1})$.
\end{proposition}

\begin{proof}
Set $n=\lceil\log_2\eta^{-1}\rceil$. Define points $f_k(t)\in X_\mu$ for $t\in T_k$ recursively. At level $1$, set $f_1(1/2)=\O_2(x,y)$. If $t\in T_k$, define
\[
    f_{k+1}(t/2)=\O_3(x,y;f_k(t)),
    \qquad
    f_{k+1}((1+t)/2)=\O_3(y,x;f_k(t)).
\]
We claim that for every $k\ge1$ and every $t\in T_k$,
\begin{equation}
\label{eq:partialobs-ruler-induction}
    |d(x,f_k(t))-t d(x,y)|\le4\eta,
    \qquad
    |d(y,f_k(t))-(1-t) d(x,y)|\le4\eta,
\end{equation}
The base case follows from the definition of $\O_2 (x, y)$. For the inductive step, suppose true for $f_k (t)$, then for $f_{k+1}(t/2)$ we have $\left|d(x,f_{k+1}(t/2))-\frac12d(x,f_k(t))\right|\le\eta$, so
\[
    \left|d(x,f_{k+1}(t/2))-\frac{td(x,y)}{2}\right|\le3\eta.
\]
Using \Cref{ora:partialobs5} and the triangle inequality then gives
\begin{multline*}
    \left|d(y,f_{k+1}(t/2))-\left(1-\frac{t}{2}\right)d(x,y)\right|\\\le\left|d(x,f_{k+1}(t/2))-\frac{td(x,y)}{2}\right| + \left| d(y,f_{k+1}(t/2))+d(x,f_{k+1}(t/2))-d(x,y) \right| \le4\eta.
\end{multline*}
The point $f_{k+1}((1+t)/2)$ is handled symmetrically. This proves \eqref{eq:partialobs-ruler-induction}. Now the maps $f_i$ induce a large map $f : 2^{-n} \Z \to X_\mu$ (we also define $f(0) = x, f(1) = y$).

For a query $(x',z)$, choose the largest $r\in2^{-n}\Z\cap[0,1]$ such that $\O_1(x',z)>\O_1(x,f(r))$. If $r=1$, then $d(x',z)>d(x,y)-\eta$, so $\left|\min\{d(x',z),d(x,y)\}-d(x,y)\right|\le\eta$. If $r<1$, then
\[
    \O_1(x,f(r+2^{-n}))\ge\O_1(x',z)>\O_1(x,f(r)).
\]
The right inequality implies $d(x',z)>d(x,f(r))-\eta$, and the left inequality implies $d(x',z)<d(x,f(r+2^{-n}))+\eta$. Using \eqref{eq:partialobs-ruler-induction} and $2^{-n}\le\eta$ gives
\[
    rd(x,y)-5\eta<d(x',z)<rd(x,y)+6\eta.
\]
Since $r\le1$, this also gives $\left|\min\{d(x',z),d(x,y)\}-rd(x,y)\right|\le6\eta$.
Taking $r$ proves the oracle. The one-time work is the construction of at most $2\eta^{-1}$ ruler points; the query cost is the direct search over these ruler points.
\end{proof}

\begin{oracle}
\label{ora:partialobs7}
Fix $0<\eta<2^{-20}$. For every $x\in X_\nu$ and $z\in X_\mu$, we can find $\O_5 (x, z) = r\in\R$ such that
\[
    |d(x,z)-r|\le\eta.
\]
The one-time work is $O(N_\nu^2+(N_\nu^2+|X_\mu|)\eta^{-1})$, and each query costs $O(\eta^{-1})$ comparisons by direct search along the ruler.
\end{oracle}

\begin{proposition}
\label{prop:partialobs5}
Assume \Cref{ora:partialobs1} and \Cref{ora:partialobs6} is available at scale $\eta$ and that $X_\nu$ is $\eta$-dense in $Y$. Then \Cref{ora:partialobs7} is available at scale $41\eta$.
\end{proposition}

\begin{proof}
Choose $x_0,y_0\in X_\nu$ maximizing $\O_1(x_0,y_0)$ over all observer pairs. By (D3) with $r=1$ and the $\eta$-density of $X_\nu$ in $Y$, there exist $x^*, y^* \in X_\nu$ with $d(x^*, y^*) \ge 1-2\eta$. Maximality and the oracle property imply $d(x_0,y_0)\ge1-3\eta$.
Apply \Cref{ora:partialobs6} with the pair $(x_0,y_0)$. For a query $(x,z)$, \Cref{ora:partialobs6} returns $r$ satisfying
\[
    |\min\{d(x,z),d(x_0,y_0)\}-rd(x_0,y_0)|\le32\eta.
\]
Since $\diam\M=1$ and $d(x_0,y_0)\ge1-3\eta$,
\[
    |d(x,z)-\min\{d(x,z),d(x_0,y_0)\}|\le3\eta.
\]
Also $d(x_0,y_0)\ge1/2$ and the preceding display imply $|r|\le2$ for $\eta<2^{-20}$. Therefore
\[
    |r-rd(x_0,y_0)|\le2(1-d(x_0,y_0))\le6\eta.
\]
Combining the last three displays gives $|d(x,z)-r|\le41\eta$.
\end{proof}

\begin{oracle}
\label{ora:partialobs8}
Fix $0<\eta<2^{-20}$. For every $x,z\in X_\mu$, we can find $\O_6 (x, z) = r\in\R$ such that
\[
    |d(x,z)-r|\le\eta.
\]
The one-time work is inherited from \Cref{ora:partialobs7}, and each query costs $O(\eta^{-1}N_\nu)$ comparisons.
\end{oracle}

\begin{proposition}
\label{prop:partialobs6}
Assume (D2), that $X_\nu$ is $\eta$-dense in $Y$, and that \Cref{ora:partialobs7} is available at scale $\eta$. Then \Cref{ora:partialobs8} is available at scale $302\eta$.
\end{proposition}

\begin{proof}
For $x,z\in X_\mu$, choose $y\in X_\nu$ maximizing
\[
    Q(y):=|\O_5(y,x)-\O_5(y,z)|.
\]
By (D2), there is $\tilde y\in Y$ with $|d(\tilde y,x)-d(\tilde y,z)|=d(x,z)$.
Choose $y^*\in X_\nu$ with $d(y^*,\tilde y)\le\eta$. Then $|d(y^*,x)-d(y^*,z)|\ge d(x,z)-2\eta$. Since \Cref{ora:partialobs7} has error at most $50\eta$, we have $Q(y^*)\ge d(x,z)-2\eta-100\eta=d(x,z)-102\eta$.
By maximality, $Q(y)\ge Q(y^*)$. Hence $|d(y,x)-d(y,z)|\ge Q(y)-100\eta\ge d(x,z)-202\eta$. 
The triangle inequality gives $|d(y,x)-d(y,z)|\le d(x,z)$. Replacing the two true distances by their two \Cref{ora:partialobs7} estimates changes their difference by at most $100\eta$, so
\[
    |d(x,z)-Q(y)|\le302\eta.
\]
The runtime is the optimization over $N_\nu$ observers, with two calls to \Cref{ora:partialobs7} for each observer.
\end{proof}

\begin{theorem}
\label{thm:partialobs-ruler-recovery}
Assume the hypotheses at the start of \Cref{sec:partial-rulers}, and assume (D2)--(D3). If \Cref{ora:partialobs0} holds at scale $\eps$, then all distances on $X_\mu$ can be recovered up to additive error
$O(\eps)$ in $O\!\left(N_\nu^2+(N_\nu^2+N_\mu)\eps^{-1}+\eps^{-1}N_\nu N_\mu^2\right)$ time after the values of \Cref{ora:partialobs0} have been computed.
\end{theorem}

\begin{proof}
Combine \Cref{ora:partialobs0} at scale $\eps$ with \Cref{prop:partialobs1}, \Cref{prop:partialobs4}, \Cref{prop:partialobs5}, and \Cref{prop:partialobs6} to get \Cref{ora:partialobs8} with $\eta = O(\eps)$.
Evaluating \Cref{ora:partialobs8} for every ordered pair in $X_\mu\times X_\mu$ gives the claimed estimates and runtime.
\end{proof}

Finally, the proof of \Cref{thm:mainresthm3} combines the ingredients in Sections \ref{sec:partial-inner-products} and \ref{sec:partial-rulers} in a similar way to proof of Theorem \ref{thm:mainresthm2}. We omit the details.

\section{Conclusion and Discussion}\label{sec:discussion}

In this paper we develop several methods which allow us to recover the pairwise distances between points in an $\eps$-net of a diameter 1 geodesic probability space satisfying certain key constraints based on different input information. Our first result generalizes the work done on \cite{noisyintrinsic} and our second addresses a line of work in boundary-value inverse problems concerning recovering distances between points on a manifold given noisy access to boundary-interior distances. Further, our methods work in a broad geodesic probability space setting, hinting at applications in non-manifold settings.

\subsection{Recovering the Scale}
\label{sec:recscale}

We now return to the topic of scaling and ask: can we guess the diameter? With certain a priori information about scale this is possible. For example, if we have access to the true distance $d(x,y)$ for a one pair of points in our $\delta$-net then the true diameter can be approximated by $\frac{d(x,y)}{\tilde{d}(x,y)}$, where $\tilde{d}(x,y)$ denotes our approximation under the assumption $\diam \M = 1$. This is reminiscent of Eratosthenes' method for measuring the circumference of the Earth by measuring the angle of the sun's rays in two cities at the same time of day, hence obtaining the angle between them along the surface of the Earth. This could be compared to the known distance between them on foot. In Eratosthenes' case he was modelling the Earth as a sphere, whereas in our case we can use our scale-invariant approximation of $\M$ as a model. In either case, we're using the shape of our space to determine its scale.

\subsection{Open Questions}
\begin{enumerate}
    \item What are the optimal sample complexity and runtime we can achieve?

    \item Is it possible to handle noisy observer-observed region distances in the setting of \cite{fefferman2025reconstruction}?

    \item Besides manifolds, what spaces are amenable to reconstruction from noisy distances using an extension of the framework in this paper?

    \item Can we extend the framework to work adaptively? A possible idea is to figure out coarse clusters, and use coarse cluster information to deduce finer clusters.
\end{enumerate}

\appendix

\section{Why this paper applies to manifolds}
\begin{proposition}\label{prop:kappa exists}
    For $d \geq 2$, let $\M$ be a $d$-dimensional Riemannian manifold with diameter $1$, $|\Sec_{\M}| \leq \Lambda^2$, and injectivity radius $\ge i_0$. Let $\mu$ be a probability measure on $M$ which is $\rho$-mutually absolutely continuous with the manifold volume measure (i.e. the Radon-Nikodym derivative is in $[\rho, 1/\rho]$). Then:
    \begin{enumerate}[(a)]
        \item $\mu(B(x, r)) \ge c(\Lambda, i_0, d) \rho r^d$ for all $r > 0$;
        \item for every $0<r_0\le c_*(\Lambda,d)i_0$ and all $x,y\in\M$ with $d(x,y)\le r_0/3$,
        \[
        \mu\left(\left\{z\in B(x,r_0/2):d(x,z)\ge d(y,z)+\frac{d(x,y)}4\right\}\right)
        \ge c(\Lambda,i_0,d)\rho r_0^d;
        \]
    \end{enumerate}
\end{proposition}

\begin{proof}[Proof of \Cref{prop:kappa exists}]
    (a) The volume of an $r$-ball is bounded below by $c(\Lambda, i_0, d) r^d$ by Bishop-G\"unther \cite[Theorem III.4.2]{chavel2006riemannian}.
    

    (b) Work in normal coordinates at $x$ at scale $r_0\le c_*(\Lambda,d)i_0$. If $d(x,y)\le r_0/3$, a fixed angular sector around the direction from $x$ through $y$ and a radial annulus $[r_0/4,r_0/2]$ is contained in the displayed local cone, up to changing dimensional constants. Bishop--G\"unther then gives volume at least $c(\Lambda,i_0,d)r_0^d$. This is the local version of the estimate used in the proof of Case 1 of  \cite[Proposition 3.5]{noisyintrinsic}.

\end{proof}

\section{Proofs of Several Lemmas}\label{sec:proofs}
\begin{proof}[Proof of \Cref{prop:density}]
Let $Z$ be a maximal $\frac{\eps}{2}$-separated set in $\M$; note that $|Z| \le \frac{1}{\VR(\eps/4)}$ by a packing argument. Then every $B(x, \eps)$ contains some $B(y, \frac{\eps}{2})$ for $y \in Z$. By a standard Chernoff bound, the probability that $B(y, \frac{\eps}{2})$ contains fewer than $\frac{1}{2} \VR(\frac{\eps}{2}) |Y|$ points of $Y$ is bounded by
\[
\mathbb{P}\Big( |B(y, \frac{\eps}{2}) \cap X| < \frac{1}{2} \VR(\frac{\eps}{2}) |Y| \Big) \le e^{-\frac{1}{8} \VR(\frac{\eps}{2}) |Y|}.
\]
This means $B(y, \frac{\eps}{2})$ and hence $B(x, \eps)$ contains a point of $Y$ for any $x \in \M$, which means that $Y$ is $\eps$-dense in $\M$. Use union bound to finish.
\end{proof}

\begin{proof}[Proof of \Cref{prop:hoeffding ip}]
    By Theorem \ref{thm:hoeffding} applied to the independent random variables $\E[X_i] Y_i$, we know that
    \begin{equation}\label{eqn:prob0}
\Pr\!\left(
\left| \sum_{i=1}^n (\E[X_i] Y_i - \mathbb{E}[X_i]\mathbb{E}[Y_i]) \right| \ge \frac{t}{2}
\right)
\le
2 \exp\!\left(
-\frac{ct^2}{8 \sum_{i=1}^n \E[X_i]^2 K^2}
\right) \le 2 \exp\!\left(
-\frac{ct^2}{8nK^2 L^2}
\right).
\end{equation}
By independence and definition of Orlicz norm, we have
$$\E[\exp(\frac{\sum_{i=1}^n (Y_i - \E Y_i)^2}{K^2})] \le \prod_{i=1}^n \E[\exp(\frac{(Y_i - \E Y_i)^2}{K^2})] \le 2^n.$$
Thus, by Markov's inequality, we have
\begin{equation}\label{eqn:prob1}
    \Pr\left(\sum_{i=1}^n (Y_i - \E Y_i)^2 \ge K^2 n + t\right) \le e^{-\frac{t}{K^2} -n} 2^n \le e^{-t/K^2}.
\end{equation}
Now by Minkowski's inequality (or triangle inequality for $L_2$ norm), we have
\begin{align*}
    \sum_{i=1}^n Y_i^2 &\le \left( \sqrt{\sum_{i=1}^n (Y_i - \E Y_i)^2} + \sqrt{\sum_{i=1}^n (\E Y_i)^2} \right)^2 \\
            &\le 2\left( \sum_{i=1}^n (Y_i - \E Y_i)^2 + \sum_{i=1}^n (\E Y_i)^2 \right) \\
            &\le 2\left( \sum_{i=1}^n (Y_i - \E Y_i)^2 + nL^2 \right).
\end{align*}
Combining this with \eqref{eqn:prob1}, we get
\begin{equation}\label{eqn:prob2}
\Pr\left(\sum_{i=1}^n Y_i^2 \ge 2((K^2 + L^2) n + t) \right)  \le e^{-t/K^2}.
\end{equation}
Finally, conditioned on a choice of $Y_1, \cdots, Y_n$ satisfying $\left| \sum_{i=1}^n (\E[X_i] Y_i - \mathbb{E}[X_i]\mathbb{E}[Y_i]) \right| \le \frac{t}{2}$ and $\sum_{i=1}^n Y_i^2 \le 2((K^2 + L^2) n + t)$, we have by Theorem \ref{thm:hoeffding} applied to the independent random variables $X_i Y_i$,
    \begin{align*}
\Pr\!\left(
\left| \sum_{i=1}^n (X_i Y_i - \mathbb{E}[X_i]\mathbb{E}[Y_i]) \right| \ge t \bigg| Y_1, \cdots, Y_n
\right) &\le \Pr\!\left(
\left| \sum_{i=1}^n (X_i Y_i - \mathbb{E}[X_i]Y_i) \right| \ge \frac{t}{2} \bigg| Y_1, \cdots, Y_n
\right) \\
&\le
2 \exp\!\left(
-\frac{ct^2}{8 \sum_{i=1}^n Y_i^2 K^2}
\right) \\
&\le
2 \exp\!\left(
-\frac{ct^2}{16K^2 ((K^2 + L^2) n + t)}
\right).
\end{align*}
Combining this with \eqref{eqn:prob0} and \eqref{eqn:prob2}, we get
\begin{align*}
    \Pr\!\left(
\left| \sum_{i=1}^n (X_i Y_i - \mathbb{E}[X_i]\mathbb{E}[Y_i]) \right| \ge t
\right)
&\le
2 \exp\!\left(
-\frac{ct^2}{8nK^2 L^2}
\right) + e^{-t/K^2} + 2 \exp\!\left(
-\frac{ct^2}{16K^2 (2K^2 n + t)}
\right) \\
&\le
5 \exp\!\left(
-\frac{ct^2}{16K^2 ((K^2 + L^2) n + t)}
\right).
\end{align*}
\end{proof}

\bibliographystyle{alpha}
\bibliography{refs.bib}

\newcommand{\etalchar}[1]{$^{#1}$}
\begin{thebibliography}{ACCPV21}

\bibitem[ACCPV21]{AriasCastroChannarondPelletierVerzelen2021}
E.~Arias-Castro, A.~Channarond, B.~Pelletier, and N.~Verzelen.
\newblock On the estimation of latent distances using graph distances.
\newblock {\em Electronic Journal of Statistics}, 15(1):722--747, 2021.

\bibitem[AVDC19]{ArayaDeCastro2019}
E.~Araya~Valdivia and Y.~De~Castro.
\newblock Latent distance estimation for random geometric graphs.
\newblock In {\em Advances in Neural Information Processing Systems}, volume~32, 2019.

\bibitem[Ber24]{bernstein1924}
S.~N. Bernstein.
\newblock On a modification of chebyshev's inequality and of the error formula of laplace.
\newblock In {\em Collected Works}. Academy of Sciences of the USSR, 1924.
\newblock English translation available in various probability anthologies.

\bibitem[Cha06]{chavel2006riemannian}
Isaac Chavel.
\newblock {\em Riemannian Geometry: A Modern Introduction}, volume~98.
\newblock Cambridge University Press, 2006.

\bibitem[Cre92]{anisotropy}
Kenneth~C Creager.
\newblock Anisotropy of the inner core from differential travel times of the phases pkp and pkikp.
\newblock {\em Nature}, 356(6367):309--314, 1992.

\bibitem[DGMY12]{DGMY2012}
D.~Doerr, I.~Gronau, S.~Moran, and I.~Yavneh.
\newblock Stochastic errors vs.\ modeling errors in distance based phylogenetic reconstructions.
\newblock {\em Algorithms for Molecular Biology}, 7:22, 2012.

\bibitem[DPEG01]{DohertyPisterElGhaoui2001}
L.~Doherty, K.~S.~J. Pister, and L.~El~Ghaoui.
\newblock Convex position estimation in wireless sensor networks.
\newblock In {\em Proceedings of IEEE INFOCOM}, 2001.

\bibitem[ELG{\etalchar{+}}24]{ElzanatyEtAl2024}
Ahmed Elzanaty, Jiuyu Liu, Anna Guerra, Francesco Guidi, Yi~Ma, and Rahim Tafazolli.
\newblock Near and far field model mismatch: Implications on 6{G} communications, localization, and sensing.
\newblock {\em IEEE Internet of Things Magazine}, 7(5):120--126, 2024.

\bibitem[EMP22]{EldanMikulincerPieters2022}
R.~Eldan, D.~Mikulincer, and H.~Pieters.
\newblock Community detection and percolation of information in a geometric setting.
\newblock {\em Combinatorics, Probability and Computing}, 31(6):1048--1069, 2022.

\bibitem[FIK{\etalchar{+}}20]{geometricwhitney}
Charles Fefferman, Sergei Ivanov, Yaroslav Kurylev, Matti Lassas, and Hariharan Narayanan.
\newblock Reconstruction and interpolation of manifolds. i: The geometric whitney problem.
\newblock {\em Found. Comput. Math.}, 20(5):1035–1133, October 2020.

\bibitem[FIL{\etalchar{+}}25]{fefferman2025reconstruction}
Charles Fefferman, Sergei Ivanov, Matti Lassas, Jinpeng Lu, and Hariharan Narayanan.
\newblock Reconstruction and interpolation of manifolds ii: Inverse problems with partial data for distances observations and for the heat kernel.
\newblock {\em American Journal of Mathematics}, 147(5):1331--1382, 2025.

\bibitem[FILN20]{noisyintrinsic}
Charles Fefferman, Sergei Ivanov, Matti Lassas, and Hariharan Narayanan.
\newblock Reconstruction of a riemannian manifold from noisy intrinsic distances.
\newblock {\em SIAM Journal on Mathematics of Data Science}, 2(3):770--808, 2020.

\bibitem[FMO{\etalchar{+}}13]{FournierEtAl2013}
Aime Fournier, Klaus Mosegaard, Henning Omre, Malcolm Sambridge, and Luis Tenorio.
\newblock Assessing uncertainty in geophysical problems --- introduction.
\newblock {\em Geophysics}, 78(3):WB1, 2013.

\bibitem[FMR25]{FMR2025}
Charles Fefferman, Jonathan Marty, and Kevin Ren.
\newblock Reconstruction of manifold distances from noisy observations.
\newblock {\em arXiv preprint arXiv:2511.13025}, 2025.

\bibitem[FS23]{FaucherScherzer2023}
Florian Faucher and Otmar Scherzer.
\newblock Quantitative inverse problem in visco-acoustic media under attenuation model uncertainty.
\newblock {\em Journal of Computational Physics}, 472:111685, 2023.

\bibitem[HJM24]{huang2024reconstructing}
Han Huang, Pakawut Jiradilok, and Elchanan Mossel.
\newblock Reconstructing the geometry of random geometric graphs.
\newblock In {\em The Thirty Seventh Annual Conference on Learning Theory}, pages 2519--2521. PMLR, 2024.

\bibitem[HJM25]{huang2025reconstructing}
Han Huang, Pakawut Jiradilok, and Elchanan Mossel.
\newblock Reconstructing riemannian metrics from random geometric graphs.
\newblock {\em arXiv preprint arXiv:2511.05434}, 2025.

\bibitem[HJM26a]{huang2026denoising}
Han Huang, Pakawut Jiradilok, and Elchanan Mossel.
\newblock Denoising distances beyond the volumetric barrier, 2026.

\bibitem[HJM26b]{huang2026denoisingmetric}
Han Huang, Pakawut Jiradilok, and Elchanan Mossel.
\newblock Denoising distances in metric measure spaces, 2026.

\bibitem[Hos12]{hoskins2012principles}
Peter~R Hoskins.
\newblock Principles of ultrasound elastography.
\newblock {\em Ultrasound}, 20(1):8--15, 2012.

\bibitem[HRH02]{HoffRafteryHandcock2002}
P.~D. Hoff, A.~E. Raftery, and M.~S. Handcock.
\newblock Latent space approaches to social network analysis.
\newblock {\em Journal of the American Statistical Association}, 97:1090--1098, 2002.

\bibitem[HRP08]{HighamRasajskiPrzulj2008}
D.~J. Higham, M.~Ra{\v{s}}ajski, and N.~Pr\v{z}ulj.
\newblock Fitting a geometric graph to a protein-protein interaction network.
\newblock {\em Bioinformatics}, 24(8):1093--1099, 2008.

\bibitem[IKL{\etalchar{+}}25]{IlmavirtaEtAl2025}
Joonas Ilmavirta, Antti Kykk\"anen, Matti Lassas, Teemu Saksala, and Sam Shedlock.
\newblock Lipschitz stability of travel time data, 2025.

\bibitem[KKL07]{KatsudaKurylevLassas2007}
A.~Katsuda, Y.~Kurylev, and M.~Lassas.
\newblock Stability of boundary distance representation and reconstruction of {R}iemannian manifolds.
\newblock {\em Inverse Problems and Imaging}, 1(1):135--157, 2007.

\bibitem[LSU03]{LassasSharafutdinovUhlmann2003}
M.~Lassas, V.~Sharafutdinov, and G.~Uhlmann.
\newblock Semiglobal boundary rigidity for {R}iemannian metrics.
\newblock {\em Mathematische Annalen}, 325:767--793, 2003.

\bibitem[Mic81]{Michel1981}
R.~Michel.
\newblock Sur la rigidit\'e impos\'ee par la longueur des g\'eod\'esiques.
\newblock {\em Inventiones Mathematicae}, 65:71--83, 1981.

\bibitem[QAN{\etalchar{+}}15]{qu2015bent}
Xiaolei Qu, Takashi Azuma, Hirofumi Nakamura, Haruka Imoto, Satoshi Tamano, Shu Takagi, Shin-Ichiro Umemura, Ichiro Sakuma, and Yoichiro Matsumoto.
\newblock Bent ray ultrasound tomography reconstruction using virtual receivers for reducing time cost.
\newblock In {\em Medical Imaging 2015: Ultrasonic Imaging and Tomography}, volume 9419, pages 75--80. SPIE, 2015.

\bibitem[SUV16]{stefanov2016boundary}
Plamen Stefanov, Gunther Uhlmann, and Andras Vasy.
\newblock Boundary rigidity with partial data.
\newblock {\em Journal of the American Mathematical Society}, 29(2):299--332, 2016.

\bibitem[Ver18]{vershynin}
Roman Vershynin.
\newblock {\em High-Dimensional Probability: An Introduction with Applications in Data Science}.
\newblock Cambridge Series in Statistical and Probabilistic Mathematics. Cambridge University Press, 2018.

\end{thebibliography}

\end{document}